\documentclass[11 pt]{article}

\usepackage[utf8]{inputenc}
\usepackage[T1]{fontenc}
\usepackage[english]{babel}
\usepackage[a4paper]{geometry}

\usepackage
[
	algo2e,
	ruled,
	vlined,
	linesnumbered
]
{algorithm2e}
\usepackage{dsfont}
\usepackage{todonotes}
\usepackage{booktabs}
\usepackage{footnote}
\usepackage[shortcuts]{extdash}
\makesavenoteenv{table}
\makesavenoteenv{tabular}

\usepackage[numbers,sort&compress]{natbib}
\usepackage{pdflscape}
\usepackage{afterpage}
\usepackage{amsmath, amsthm, amssymb, mathtools}

\newtheorem{theorem}{Theorem}
\newtheorem{lemma}{Lemma}
\newtheorem{corollary}{Corollary}

\usepackage[affil-it]{authblk}
\usepackage{titling}

\definecolor{darkblue}{rgb}{0, 0, 0.5}

\usepackage
[
    bookmarks = true,
    bookmarksopen = false,
    bookmarksnumbered = true,
    pdfstartpage = 1,
    breaklinks = true,
    colorlinks = true,
    linkcolor = darkblue,
    anchorcolor = darkblue,
    citecolor = darkblue,
    filecolor = darkblue,
    menucolor = darkblue,
    pagecolor = darkblue,
    urlcolor = darkblue
]
{hyperref}

\newcommand*{\N}{\mathds{N}}
\newcommand*{\R}{\mathds{R}}
\newcommand*{\OM}{\textsc{OM}\xspace}
\newcommand*{\LO}{\textsc{LO}\xspace}
\newcommand*{\BV}{\textsc{B\!V}\xspace}
\newcommand*{\OneMax}{\textsc{OneMax}\xspace}
\newcommand*{\LeadingOnes}{\textsc{LeadingOnes}\xspace}
\newcommand*{\BinVal}{\textsc{BinVal}\xspace}
\newcommand*{\Var}{\mathrm{Var}}
\newcommand*{\Bin}{\mathrm{Bin}}
\newcommand*{\sig}{\mathrm{sig}_\varepsilon}
\newcommand*{\up}{\mathrm{up}}
\newcommand*{\down}{\mathrm{down}}
\newcommand*{\stay}{\mathrm{stay}}
\newcommand*{\sigcGA}{sig\=/cGA\xspace}

\let\originalleft\left
\let\originalright\right
\renewcommand{\left}{\mathopen{}\mathclose\bgroup\originalleft}
\renewcommand{\right}{\aftergroup\egroup\originalright}

\makeatletter
\def\NAT@spacechar{~}
\makeatother

\allowdisplaybreaks[1]

\begin{document}

\title{Significance-based Estimation-of-Distribution Algorithms}

\author{Benjamin Doerr\thanks{Laboratoire d'Informatique (LIX), {\'E}cole Polytechnique, Palaiseau, France\\ \hspace*{0.8 em}\emph{e-mail: \href{mailto:doerr@lix.polytechnique.fr}{doerr@lix.polytechnique.fr}}}\hspace*{0.5 em} and Martin S. Krejca\thanks{Hasso Plattner Institute, University of Potsdam, Potsdam, Germany\\ \hspace*{0.8 em}\emph{e-mail: \href{mailto:martin.krejca@hpi.de}{martin.krejca@hpi.de}}}}
\date{\vspace*{-1 cm}}

\maketitle

\begin{abstract}
	\boldmath
	\emph{Estimation-of-distribution algorithms} (EDAs) are randomized search heuristics that create a probabilistic model of the solution space, which is updated iteratively, based on the quality of the solutions sampled according to the model. As previous works show, this iteration-based perspective can lead to erratic updates of the model, in particular, to bit-frequencies approaching a random boundary value.
	In order to overcome this problem, we propose a new EDA based on the classic compact genetic algorithm (cGA) that takes into account a longer history of samples and updates its model only with respect to information which it classifies as statistically significant. We prove that this \emph{significance-based compact genetic algorithm} (\sigcGA) optimizes the commonly regarded benchmark functions \OneMax, \LeadingOnes, and \BinVal all in quasilinear time, a result shown for no other EDA or evolutionary algorithm so far. %
	For the recently proposed scGA~-- an EDA that tries to prevent erratic model updates by imposing a bias to the uniformly distributed model~-- we prove that it optimizes \OneMax only in a time exponential in its hypothetical population size. Similarly, we show that the convex search algorithm cannot optimize \OneMax in polynomial time.
\end{abstract}

\section{Introduction}

\emph{Estimation-of-distribution algorithms} (EDAs;~\cite{PelikanHandbook15}) are nature-inspired heuristics, similar to \emph{evolutionary algorithms} (EAs). In contrast to EAs, which maintain an explicit set of solutions, EDAs optimize a function by evolving a probabilistic model of the solution space. Iteratively, an EDA uses its probabilistic model in order to generate samples and make observations from them. It then updates its model based on these observations, where an algorithm-specific parameter determines how strong the changes to the model in each iteration are.

For an EDA to succeed in optimization, it is important that its model is changed over time in a way that better solutions are sampled more frequently. However, due to the randomness in sampling, the model should not be changed too drastically in a single iteration in order to prevent wrong updates from having a long-lasting impact.

The theoretical analysis of EDAs has recently gained momentum (see, e.g., the survey~\cite{KrejcaW18EDABookChapter})
and has clearly shown that this trade-off between convergence speed and accumulation of erratic updates can be delicate and non-trivial to understand. %
Among the most relevant works, Sudholt and Witt~\cite{SudholtW19cGAACOOneMaxLowerBound} and Krejca and Witt~\cite{DBLP:conf/foga/KrejcaW17} prove lower bounds of the expected run times of three common EDAs on the benchmark function \OneMax. In simple words, these bounds show that if the update parameter for the model is too large, the model converges too quickly and very likely to a wrong model; consequently, it then takes a long time to find the optimum. On the other hand, if the parameter is too small, then the model converges to the correct model but does so slowly. More formally, Sudholt and Witt~\cite{SudholtW19cGAACOOneMaxLowerBound} prove a lower bound of $\Omega(K\sqrt{n} + n \log n)$ for the $2$-MMAS$_\textrm{ib}$ and the cGA, where~$1/K$ is the step size of the algorithm, and Krejca and Witt~\cite{DBLP:conf/foga/KrejcaW17} prove a lower bound of $\Omega(\lambda\sqrt{n} + n \log n)$ for the UMDA, where~$\lambda$ is the population size of the algorithm. These results show that choosing the parameter with a value of $\omega(\sqrt{n}\log n)$ has no benefit.
Further, it has been recently shown by Lengler et~al.~\cite{LenglerSW18cGAMediumSteps} that the run time of the cGA on \OneMax is $\Omega(K^{1/3}n + n \log n)$ for $K = O(\sqrt{n}/(\log n \log \log n))$. Together with the results from Sudholt and Witt~\cite{SudholtW19cGAACOOneMaxLowerBound}, this implies a bimodal behavior in the run time with respect to~$K$ if $K = \Omega(\log n) \cap O(\sqrt{n}\log n)$, showing that the run time is sensitive to the parameter choice.

Friedrich et~al.~\cite{DBLP:conf/gecco/FriedrichKK16} also discuss the problem of how to choose the update strength. They consider a class of EDAs optimizing functions over bit strings of length $n$ that all current theoretical results fall into, named $n$-Bernoulli-$\lambda$-EDA. The model of such EDAs uses one variable per bit of a bit string, resulting in a probability vector $\tau$ of length $n$ called the \emph{frequency vector}. In each iteration, a bit string $x$ is sampled bit-wise independently and independent of any other sample such that bit $x_i$ is $1$ with probability \emph{(frequency)} $\tau_i$ and $0$ otherwise.

Friedrich et~al.~\cite{DBLP:conf/gecco/FriedrichKK16} consider two different properties of such EDAs, namely \emph{balanced} and \emph{stable}. Intuitively, a \emph{balanced} EDA does not change a frequency $\tau_i$ in expectation if the fitness function has no preference for $0$s or $1$s at position $i$. A \emph{stable} EDA keeps a frequency, in such a scenario, close to $1/2$. Friedrich et~al.~\cite{DBLP:conf/gecco/FriedrichKK16} then prove that an $n$-Bernoulli-$\lambda$-EDA cannot be both balanced and stable. They also prove that all commonly theoretically analyzed EDAs are balanced. This means that the frequencies will always move toward $0$ or $1$, even if there is no bias from the objective function.

Motivated by these results, Friedrich et~al.~\cite{DBLP:conf/gecco/FriedrichKK16} propose an EDA (called \emph{scGA}) that is stable (but not balanced) %
by introducing an artificial bias into the update process that should counteract the bias of a balanced EDA. However, we prove that this approach fails badly on the standard benchmark function \OneMax (Thm.~\ref{thm:scGAOneMax}). We note that a similar bias towards the middle frequency of $1/2$ was proven for a binary differential evolution algorithm by Zheng et~al.~\cite{ZhengYD18}. Similar to the situation of the scGA, their run time results (partially relying on mean-field assumptions) indicate that \LeadingOnes is optimized in a number of generations that is linear in the problem size~$n$. This gives an $O(n \log n)$ number of function evaluations when using a logarithmic population size (and smaller population sizes are provably not successful). For \OneMax, the results are less conclusive, but they indicate a run time exponential in the population size can occur. 

\afterpage{
	\newgeometry{top = 2.3 cm}
	\begin{landscape}
		\pagestyle{empty}
		\enlargethispage*{2\baselineskip}
		\begin{table}
			\newgeometry{textwidth = 23.3 cm}
			\vspace*{-1 cm}
			\caption{Expected run times (number of fitness evaluations) of various algorithms until they first find an optimum for the two functions \OM (eq.~\eqref{eq:oneMax}) and \LO (eq.~\eqref{eq:leadingOnes}). For optimal parameter settings, many algorithms have a run time of $\Theta(n\log n)$ for \OM and of $\Theta(n^2)$ for \LO. We note that the $\big(1 + (\lambda, \lambda)\big)$~GA has an $o(n \log n)$ run time on \OM (and even linear run time with a dynamic parameter choice), but we do not see why it should have a performance better than quadratic on \LO. %Further, we strongly believe that the CSA has an exponential run time on \OM.
				\vspace*{3 ex}
			}
			\hspace*{-1 em}\begin{tabular}{lp{5.5 cm}p{5 cm}p{3.7 cm}p{4.9 cm}}
				\label{tab:runTimeComparison}
				Algorithm & \OM & constraints & \LO & constraints\\ \toprule
				$(1 + 1)$~EA & $\Theta(n\log n)$~\cite{DBLP:journals/tcs/DrosteJW02} & none & $\Theta(n^2)$~\cite{DBLP:journals/tcs/DrosteJW02} & none\\
				
				$(\mu + 1)$~EA & $\Theta(\mu n + n\log n)$~\cite{Witt06SimplePseudoBool} & $\mu = O\big(\mathrm{poly}(n)\big)$ & $\Theta(\mu n\log n + n^2)$~\cite{Witt06SimplePseudoBool} & $\mu = O\big(\mathrm{poly}(n)\big)$\\
				
				$(1 + \lambda)$~EA & $\Theta\Big(n\log n + \frac{\lambda n \log\log\lambda}{\log\lambda}\Big)$ \cite{JansenJW05, DBLP:journals/tcs/DoerrK15}\footnote{Better run time bounds for the $(1 + \lambda)$~EA are known if the mutation rate is (i) fitness-dependent~\cite{DBLP:conf/ppsn/BadkobehLS14}, (ii) self-adjusting~\cite{DBLP:journals/algorithmica/DoerrGWY19}, or (iii) self-adaptive~\cite{DBLP:conf/gecco/DoerrWY18}.} & $\lambda = O(n^{1 - \varepsilon})$ & $\Theta(n^2 + \lambda n)$~\cite{JansenJW05} & $\lambda = O\big(\mathrm{poly}(n)\big)$\\
				
				$(\mu + \lambda)$~EA & $\Theta\Big(\frac{n \log n}{\lambda} + \frac{n}{\lambda/\mu} + \frac{n\log^+\log^+ \lambda/\mu}{\log^+ \lambda/\mu}\Big)$~\cite{AntipovDFH18muLambdaEA} & $\log^+ x \coloneqq \max\{1, \log x\}$ & $\Omega\left(n^2 + \frac{\lambda n}{\log(\lambda/n)}\right)$~\cite{DBLP:conf/ppsn/BadkobehLS14} & --\\
				
				$\big(1 + (\lambda, \lambda)\big)$~GA & $\Theta\Big(\!\!\max\!\Big\{\frac{n\log n}{\lambda}, \!\frac{n\lambda\log\log\lambda}{\log\lambda}\Big\}\Big)$ \cite{DBLP:journals/algorithmica/DoerrD18} & $p = \frac{\lambda}{n}$, $c = \frac{1}{\lambda}$ & unknown & --\\
				
				CSA & $\Omega(n^c)$ [Thm.~\ref{thm:csa}] & $c > 0$ & $O(n\log n)$~\cite{DBLP:journals/ec/MoraglioS17} & $\mu \geq 8\ln\!\big((4n + 6)n\big)$, restarts\\
				
				UMDA/PBIL\footnote{The results shown for PBIL are the results of UMDA if not mentioned otherwise, since the latter is a special case of the former. %
				} & $\Omega(\lambda\sqrt{n} + n\log n)$ \cite{DBLP:conf/foga/KrejcaW17} & $\mu = \Theta(\lambda)$ & $O(n\lambda\log\lambda + n^2)$~\cite{DangLN19UMDAOneMaxUpperBound, LehreN18PBIL} & $\lambda = \Omega(\log n)$, $\mu = \Theta(\lambda)$\\
				
				& $O(\lambda n)$ \cite{DBLP:journals/algorithmica/Witt19, DangLN19UMDAOneMaxUpperBound} & $\mu = \Omega(\log n) \cap O(\sqrt{n})$, $\lambda = \Omega(\mu)$ or & &\\
				& & $\mu = \Omega(\sqrt{n}\log n)$, $\mu = \Theta(\lambda)$ or & &\\
				& & $\mu = \Omega(\log n) \cap o(n)$, $\mu = \Theta(\lambda)$ & &\\
				
				cGA/2-MMAS$_\textrm{ib}$ & $\Omega\Big(\frac{\sqrt{n}}{\rho} + n\log n\Big)$ \cite{SudholtW19cGAACOOneMaxLowerBound} & $\frac{1}{\rho} = O\big(\mathrm{poly}(n)\big)$ & unknown & --\\
				
				& $O\Big(\frac{\sqrt{n}}{\rho}\Big)$ \cite{SudholtW19cGAACOOneMaxLowerBound} & $\frac{1}{\rho} = \Omega(\sqrt{n}\log n) \cap O\big(\mathrm{poly}(n)\big)$ & &\\
				
				1-ANT & $O(n^2)$~\cite{DBLP:journals/algorithmica/NeumannW09}\footnote{For $\rho \geq (n - 2)/(3n - 2)$, the algorithm simulates the $(1+1)$~EA and has a run time of $\Theta(n \log n)$.} & $\rho = \Omega(n^{-1 + \varepsilon})$ & $O(n^2 \cdot (6e)^{1/(n\rho)})$ \cite{DoerrNSW11OneAnt} & none\\
				
				& & & $2^{\Omega(\min\{n, 1/(n\rho)\})}$ \cite{DoerrNSW11OneAnt} & none\\
				
				MMAS$^*$ & $O\big(\frac{n\log n}{\rho}\big)$~\cite{DBLP:journals/swarm/NeumannSW09} & $\rho = O(1)$ & $O\big(n^2 + \frac{n\log n}{\rho}\big)$~\cite{DBLP:journals/swarm/NeumannSW09} & $\rho = O(1)$\\
				
				& & & $\Omega\big(n^2 + \frac{n}{-\rho\log(2\rho)}\big)$~\cite{DBLP:journals/swarm/NeumannSW09} & $\rho = 1/\mathrm{poly}(n)$\\
				
				scGA & $\Omega\big(\!\min\{2^{\Theta(n)}, 2^{c/\rho}\}\big)$ [Thm.~\ref{thm:scGAOneMax}] & $1/\rho\! =\! \Omega(\log n)$, $a = \Theta(\rho)$, $d = \Theta(1)$, $c > 0$ & $O(n\log n)$~\cite{DBLP:conf/gecco/FriedrichKK16} & $1/\rho = \Theta(\log n)$, $a = O(\rho)$, $d = \Theta(1)$\\
				
				\sigcGA (Alg.~\ref{alg:adaptiveEDA}) & $O(n\log n)$ [Thm.~\ref{thm:oneMax}] & $\varepsilon > 12$ & $O(n\log n)$ [Thm.~\ref{thm:leadingOnes}] & $\varepsilon > 12$
			\end{tabular}
			\label{tab:runtimes}
			\vspace*{-4 cm}
		\end{table}
	\end{landscape}
}

The results of Friedrich et~al.~\cite{DBLP:conf/gecco/FriedrichKK16}, Sudholt and Witt~\cite{SudholtW19cGAACOOneMaxLowerBound}, Krejca and Witt~\cite{DBLP:conf/foga/KrejcaW17}, and Lengler et~al.~\cite{LenglerSW18cGAMediumSteps} draw the following picture: for a balanced EDA, there exists some inherent noise in the update. Thus, if the parameter responsible for the update of the probabilistic model is large and the speed of convergence high, the algorithm only uses a few samples before it converges. During this time, the noise introduced by the balance-property may not be overcome, resulting in the probabilistic model converging to an incorrect one, as the algorithms are not stable. Hence, the parameter has to be chosen sufficiently small in order to guarantee convergence to the correct model, resulting in a slower optimization time.

As we shall argue in this work, the reason for this dilemma is that EDAs only use information from a single iteration when performing an update. Thus, the decision of whether and how a frequency should be changed has to be made on the spot, which may result in harmful decisions.

To overcome these difficulties, we propose a conceptually new EDA that has some access to the search history and updates the model only if there is sufficient reason. The \emph{significance-based compact genetic algorithm} (\sigcGA) stores for each position the history of bits of good solutions so far. If it detects that either statistically significantly more $1$s than $0$s or vice versa were sampled, it changes the corresponding frequency, otherwise not. Thus, the \sigcGA only performs an update when it has proof that it makes sense. This sets it apart from the other EDAs analyzed so far. 

We prove that the \sigcGA is able to optimize \LeadingOnes, \OneMax, and \BinVal in $O(n\log n)$ function evaluations in expectation and with high probability (Thms.~\ref{thm:leadingOnes} and~\ref{thm:oneMax} and Cor.~\ref{cor:binVal}), which has not been proven before for any other EDA or classical EA (for further details, see Table~\ref{tab:runTimeComparison}).

We also observe that the analysis for \LeadingOnes can easily~be~modified to also show an $O(n \log n)$ run time for the \emph{binary value} function \BinVal, which is a linear function with exponentially growing coefficients. This result is interesting in that it indicates that the \sigcGA has asymptotically the same run time on \BinVal and \OneMax. %
In contrast, for the classic cGA it is known~\cite{DBLP:journals/nc/Droste06} that the run times on \OneMax and \BinVal differ significantly.

We then show that two previously regarded algorithms which solve \LeadingOnes in $O(n \log n)$ time behave poorly on \OneMax. The run time of the scGA proposed in~\cite{DBLP:conf/gecco/FriedrichKK16} is $\Omega(2^{\Theta(\min\{n,1/\rho\})})$ (Thm.~\ref{thm:scGAOneMax}), where $1/\rho$ is an algorithm-specific parameter controlling the strength of the model update and denotes the hypothetical population size of the algorithm. 
For the convex search algorithm (CSA) proposed in~\cite{DBLP:journals/ec/MoraglioS17}, we prove that the run time, even when adding suitable restart schemes, is asymptotically larger than any polynomial (Thm.~\ref{thm:csa}). These results, together with the large number of existing results, suggest that none of the previously known algorithms performs exceptionally well on both \OneMax and \LeadingOnes.

These results, the positive ones for the \sigcGA using a longer history of the search process and the negative ones for other algorithms not exploiting a longer history, suggest that a fruitful direction for the future development of the field of evolutionary computation (EC; not restricted to theory) is the search for algorithms that enrich the classic generational approaches with mechanisms that profit from regarding more than one generation. We discuss this in more detail in the conclusions of this paper. We note that, from a practical point of view, our algorithm not only showed a performance not seen so far with other algorithms, it is also easier to use since, unlike with most other EDAs, the delicate choice of the update strength is obsolete.

This paper extends our previous results on the \sigcGA~\cite{DoerrK18sigcGA} by proving an upper bound of the \sigcGA on~\BinVal (Cor.~\ref{cor:binVal}) and a lower bound of the CSA on~\OneMax (Thm.~\ref{thm:csa}).

\section{Preliminaries}
\label{sec:preliminaries}

In this paper, we consider the maximization of pseudo-Boolean functions $f\colon \{0, 1\}^n \to \R$, where $n$ is a positive integer (fixed for the remainder of this work). We call $f$ a \emph{fitness function}, an element $x \in \{0, 1\}^n$ an \emph{individual}, and, for an $i \in [n] \coloneqq [1, n] \cap \N$, we denote the $i$th bit of $x$ by $x_i$. %
When talking about \emph{run time}, we always mean the number of fitness function evaluations of an algorithm until an optimum is sampled for the first time.

In our analysis, we regard the two classic benchmark functions \OneMax ($\OM$) and \LeadingOnes ($\LO$) defined~by
\begin{align}
	\label{eq:oneMax}
	\OM(x) &= \sum_{i \in [n]} x_i \hspace*{2 em}\textrm{and}\\
	\label{eq:leadingOnes}
	\LO(x) &= \sum_{i \in [n]} \prod_{j \in [i]} x_j\ .
\end{align}
Intuitively, \OM returns the number of $1$s of an individual, whereas \LO returns the longest sequence of consecutive $1$s, starting from the left. Note that the all-$1$s bit string is the unique global optimum for both functions.

In Table~\ref{tab:runtimes}, we state the asymptotic run times of many algorithms on these two functions. We note that (i)~the black-box complexity of \OM is $\Theta(n / \log n)$, see~\cite{DrosteJW06,AnilW09}, and (ii)~the black-box complexity of \LO is $\Theta(n \log\log n)$, see~\cite{AfshaniADDLM13}, however, all black-box algorithms witnessing these run times are highly artificial. Consequently, $\Theta(n \log n)$ appears to be the best run time to aim for these two problems.

For our calculations, we shall regularly use the following well-known variance-based additive Chernoff bounds (see, e.g., the respective Chernoff bound in~\cite{2018arXiv180106733D}).
\begin{theorem}[Variance-based Additive Chernoff Bounds]
	\label{thm:additiveChernoff}
	Let $X_1, \ldots, X_n$ be independent random variables such that, for all $i \in [n]$, $E[X_i] - 1 \leq X_i \leq E[X_i] + 1$. Further, let $X = \sum_{i = 1}^{n} X_i$ and $\sigma^2 = \sum_{i = 1}^{n} \Var[X_i] = \Var[X]$. Then, for all $\lambda \geq 0$, abbreviating $m = \min\{\lambda^2/\sigma^2, \lambda\}$,
	\[
	\Pr[X \geq E[X] + \lambda] \leq e^{-\frac{1}{3}m} \textrm{ and } \Pr[X \leq E[X] - \lambda] \leq e^{-\frac{1}{3}m}\ .
	\]
\end{theorem}

We say that an event $A$ occurs with high probability (w.h.p.) if there is a $c = \Omega(1)$ such that $\Pr[A] \geq (1 - n^{-c})$.

Last, we use the $\circ$ operator to denote string concatenation. For a bit string $H \in \{0, 1\}^*$, let $|H|$ denote its length, $\|H\|_0$ its number of $0$s, $\|H\|_1$ its number of $1$s, and, for a $k \in [|H|]$, let $H[k]$ denote the \emph{last} $k$ bits in $H$. In addition to that, $\emptyset$ denotes the empty string.

\section{The Significance-based Compact Genetic Algorithm}
\label{sec:sigcGA}

\begin{algorithm2e}[t]
	\caption{The \sigcGA with parameter $\varepsilon$ and significance function~$\sig$ (eq.~\eqref{eq:updateFunction}) optimizing~$f$}
	\label{alg:adaptiveEDA}
	$t \gets 0$\;
	\lFor{$i \in [n]$}
	{
		$\tau^{(t)}_i \gets \frac{1}{2}$ and $H_i \gets \emptyset$
	}
	\Repeat{\emph{termination criterion met}}
	{
		$x, y \gets$ offspring sampled with respect to $\tau^{(t)}$\;
		$x \gets$ winner of $x$ and $y$ with respect to $f$\;
		\For{$i \in [n]$}
		{
			$H_i \gets H_i \circ x_i$\;
			\lIf{$\sig(\tau^{(t)}_i, H_i) = \up$}{$\tau^{(t + 1)}_i \gets 1 - 1/n$}
			\lElseIf{$\sig(\tau^{(t)}_i, H_i) = \down$}{$\tau^{(t + 1)}_i \gets 1/n$}
			\lElse{$\tau^{(t + 1)}_i \gets \tau^{(t)}_i$}
			\lIf{$\tau^{(t + 1)}_i \neq \tau^{(t)}_i$}{$H_i \gets \emptyset$}
		}
		$t \gets t + 1$\;
	}
\end{algorithm2e}

Before we present our algorithm \sigcGA in detail in Section~\ref{subsec:sigcGA}, we provide more information about the \emph{compact genetic algorithm} (cGA~\cite{HarikLG98}), which the \sigcGA as well as the scGA are based on.

The cGA is a univariate EDA, that is, it assumes independence of the bits in the search space. As such, it keeps a vector of probabilities $(\tau_i)_{i \in [n]}$ (the \emph{frequency vector}). In each iteration, two individuals \emph{(offspring)} are sampled in the following way with respect to~$\tau$: for an individual $x \in \{0, 1\}^n$, we have $x_i = 1$ with probability $\tau_i$, and $x_i = 0$ with probability $1 - \tau_i$, independently of any $\tau_j$ with $j \neq i$.

After sampling, the frequency vector is updated with respect to a fitness-based ranking of the offspring. The process of choosing how the offspring are ranked is called \emph{selection}. Let $x$ and $y$ denote both offspring of the cGA during an iteration. Given a fitness function~$f$, we rank $x$ above $y$ if $f(x) > f(y)$ (as we maximize), and we rank $y$ above $x$ if $f(y) > f(x)$. If $f(x) = f(y)$, we rank them randomly. The higher-ranked individual is called the \emph{winner}, the other individual the \emph{loser}. Assume that $x$ is the winner. The cGA then changes a frequency $\tau_i$ then with respect to the difference $x_i - y_i$ by a value of $\rho$ (where $1/\rho$ is usually referred to as population size). Hence, no update is performed if the bit values are identical, and the frequency is moved to the bit value of the winner. In order to prevent a frequency $\tau_i$ getting stuck at $0$ or $1$,\!\footnote{A frequency $\tau_i$ at one of these two values results in the offspring only having the same bit value at position $i$. Thus, the cGA would not change $\tau_i$ anymore.} the cGA usually caps its frequency to the range $[1/n, 1 - 1/n]$, as is common practice. This way, a frequency can get close to $0$ or $1$, but it is always possible to sample $0$s and $1$s.

Consider a position $i$ and any two individuals $x$ and $y$ that are identical except for position $i$. Assume that $x_i > y_i$. If the probability that $x$ is the winner of the selection is higher than $y$ being the winner, we speak of a \emph{bias in selection} (for $1$s) at position $i$. Analogously, we speak of a bias for $0$s if the probability that $y$ wins is higher than the probability that $x$ wins. Usually, a fitness function introduces a bias into the selection and thus into the update.

\subsection{Detailed Description of the \sigcGA}
\label{subsec:sigcGA}

Similar to the cGA, our new algorithm~-- the \emph{significance-based compact genetic algorithm} (\sigcGA; Alg.~\ref{alg:adaptiveEDA})~-- also samples two offspring each iteration. However, in contrast to the cGA, it keeps a history of bit values for each position and only performs an~update when a statistical significance within a history occurs. This~approach better aligns with the intuitive reasoning that an update should only be performed if there is valid evidence for a different frequency being better suited for sampling good individuals.

In more detail, for each bit position $i \in [n]$, the \sigcGA keeps a history $H_i \in \{0, 1\}^*$ of all the bits sampled by the winner of each iteration since the last time $\tau_i$ changed -- the last bit denoting the latest entry. Observe that if there is no bias in selection at position $i$, the bits sampled by $\tau_i$ follow a binomial law with $|H_i|$ tries and a success probability of $\tau_i$. We call this our \emph{hypothesis}. If we happen to find a sequence (starting from the latest entry) in $H_i$ that significantly deviates from the hypothesis, we update $\tau_i$ with respect to the bit value that occurred significantly, and we reset the history. We only use the following three frequency values:
\begin{itemize}
	\item $1/2$: starting value;
	
	\item $1/n$: significance for $0$s was detected;
	
	\item $1 - 1/n$: significance for $1$s was detected.
\end{itemize}

We formalize \emph{significance} by defining the threshold~$s$ to overcome. For all $\varepsilon, \mu \in \R^+$, where $\mu$ is the expected value of our hypothesis and $\varepsilon$ is an algorithm-specific parameter:
\[
s(\varepsilon, \mu) = \varepsilon\max\!\big\{\sqrt{\mu\ln n}, \ln n\big\}\ .
\]
Note that $\sqrt{\mu}$ basically describes the standard deviation of our hypothesis, and the logarithmic factor increases this value such that a deviation does not happen w.h.p. The maximum ensures that we consider at least logarithmically many samples before we conclude that we found a significance, eliminating wrong updates due to small samples sizes w.h.p. The parameter~$\varepsilon$ effectively turns into the exponent of the w.h.p. bounds. Thus, a larger value of~$\varepsilon$ decreases the probability of detecting a false significance by a polynomial amount. However, it also increases the number of samples necessary in order to change a frequency. This results in a linear factor of~$\varepsilon$ in the run time. We provide more details on how~$\varepsilon$ should be chosen at the end of this subsection (after Cor.~\ref{cor:noFrequenciesDrop}).

We say, for an $\varepsilon \in \R^+$, that a binomially distributed random variable $X$ deviates significantly from a hypothesis $Y \sim \mathrm{Bin}(k, \tau)$, where $k \in \N^+$ and $\tau \in [0, 1]$, if there exists a $c = \Omega(1)$ such that
$
\Pr\!\big[|X - E[Y]| \leq s(\varepsilon, E[Y])\big] \leq n^{-c}\ .
$

We now state our significance function $\sig\colon \big\{\frac{1}{n}, \frac{1}{2}, 1 - \frac{1}{n}\big\} \times \{0, 1\}^* \to \{\up, \stay, \down\}$, which scans a history for a significance. However, it does not scan the entire history but multiple subsequences of a history (always starting from the latest entry). This is done in order to quickly notice a change from an insignificant history to a significant one. Further, we only check in steps of powers of $2$, as this is faster than checking each subsequence and we can be off from any length of a subsequence by a constant factor of at most $2$.
More formally, for all $p \in \big\{\frac{1}{n}, \frac{1}{2}, 1 - \frac{1}{n}\big\}$ and all $H \in \{0, 1\}^*$, we define, with $\varepsilon$ being a parameter of the \sigcGA, recalling that $H[k]$ denotes the last $k$~bits of~$H$,
\begin{align}
	\label{eq:updateFunction}
	\sig(p, H) =
	\begin{cases}
		\up &\textrm{if } p \in \{\tfrac{1}{n}, \tfrac{1}{2}\} \land \exists m \in \N\colon\\
		&\hspace*{1.5 em} \|H[2^m]\|_1 \geq 2^m p + s\big(\varepsilon, 2^m p\big),\\
		\down &\textrm{if } p \in \{\tfrac{1}{2}, 1 - \tfrac{1}{n}\} \land \exists m \in \N\colon\\
		&\hspace*{1.5 em}\|H[2^m]\|_0 \geq 2^m (1 - p) + s\big(\varepsilon, 2^m (1 - p)\big),\\
		\stay &\textrm{else.}
	\end{cases}
\end{align}
We stop at the first (minimum) length $2^m$ that yields a significance. Thus, we check a history $H$ in each iteration at most $\log_2 |H|$ times.

We now prove that the \sigcGA does not detect a significance at a position with no bias in selection (i.e., a \emph{false significance}) w.h.p.
\begin{lemma}
	\label{lem:unbiasedFalsePositives}
	Consider the \sigcGA (Alg.~\ref{alg:adaptiveEDA}) with $\varepsilon \geq 1$. Further, consider a position $i \in [n]$ and an iteration such that the distribution $X$ of $1$s of $H_i$ follows a binomial law with $k$ tries and success probability $\tau_i$, i.e., there is no bias in selection at position $i$. Then $\tau_i$ changes in this iteration with a probability of at most $n^{-\varepsilon/3}\log_2 k$.
\end{lemma}

\begin{proof}
	In order for $\tau_i$ to change, the number of $0$s or $1$s in $X$ needs to deviate significantly from the hypothesis, which follows the same distribution as $X$ by assumption. We are going to use Theorem~\ref{thm:additiveChernoff} in order to show that, in such a scenario, $X$ will deviate significantly from its expected value only with a probability of at most $n^{-\varepsilon/3}\log_2 k$ for any number of trials at most $k$.
	
	Let $\tau'_i = \min\{\tau_i, 1 - \tau_i\}$. Note that, in order for $\tau_i$ to change, a significance of values sampled with probability $\tau'_i$ needs to be sampled. That is, for $\tau_i = 1/2$, either a significant amount of $1$s or $0$s needs to occur; for $\tau_i = 1 - 1/n$, a significant amount of $0$s needs to occur; and, for $\tau_i = 1/n$, a significant amount of $1$s needs to occur. Further, let $X'$ denote the number of values we are looking for a significance within $k' \leq k$ trials. That is, if $\tau_i = 1/2$, $X'$ is either the number of $1$s or $0$s; if $\tau_i = 1 - 1/n$, $X'$ is the number of $0$s; and if $\tau_i = 1/n$, $X'$ is the number of $1$s.
	
	Given the definition of $\tau'_i$, we see that $E[X'] = k'\tau'_i$ and $\Var[X'] = k'\tau_i(1 - \tau_i) \leq k'\tau'_i$. Since we want to apply Theorem~\ref{thm:additiveChernoff}, let $\lambda = s(\varepsilon, E[X']) = s(\varepsilon, k'\tau'_i)$ and $\sigma^2 = \Var[X']$.
	
	First, consider the case that $\lambda = s(\varepsilon, k'\tau'_i) = \varepsilon\ln n$, i.e., that $(k'\tau'_i \ln n)^{1/2} \leq \ln n$, which is equivalent to $k' \leq (1/\tau'_i)\ln n$. Note that $\lambda^2/\sigma^2 \geq \varepsilon^2\ln n \geq \ln n$, as $\varepsilon \geq 1$. Thus, $\min\{\lambda^2/\sigma^2, \lambda\} \geq \varepsilon\ln n$.
	
	Now consider the case $\lambda = s(\varepsilon, k'\tau'_i) = \varepsilon(k'\tau'_i\ln n)^{1/2}$, i.e., that $(k'\tau'_i \ln n)^{1/2} \geq \ln n$, which is equivalent to $k' \geq (1/\tau'_i)\ln n$. We see that $\lambda \geq \varepsilon\ln n$ and $\lambda^2/\sigma^2 \geq \varepsilon^2\ln n$. Hence, as before, we get $\min\{\lambda^2/\sigma^2, \lambda\} \geq \varepsilon\ln n$.
	
	Combining both cases and applying Theorem~\ref{thm:additiveChernoff}, we get
	\begin{align*}
		\Pr[X' \geq k'\tau'_i + s(\varepsilon, k'\tau'_i)] &= \Pr[X' \geq E[X'] + \lambda]\\
		&\hspace*{-2 em}\leq e^{-\frac{1}{3}\min\left\{\frac{\lambda^2}{\sigma^2}, \lambda\right\}} \leq e^{-\frac{\varepsilon}{3}\ln n} = n^{-\frac{\varepsilon}{3}}\ .
	\end{align*}
	That is, the probability of detecting a (false) significance during $k'$ trials is at most $n^{-\varepsilon/3}$. Since we look for a significance a total of at most $\log_2 k$ times during an iteration, we get by a union bound that the probability of detecting a significance within a history of length $k$ is at most $n^{-\varepsilon/3}\log_2 k$.
\end{proof}

Lemma~\ref{lem:unbiasedFalsePositives} bounds the probability of detecting a false significance within a single iteration, assuming no bias in selection. The following corollary trivially bounds the probability of detecting a false significance within any number of iterations.

\begin{corollary}
	\label{cor:noFrequenciesDrop}
	Consider the \sigcGA (Alg.~\ref{alg:adaptiveEDA}) with $\varepsilon \geq 1$ running for $k$ iterations such that, during each iteration, for each $i \in [n]$, a $1$ is added to $H_i$ with probability $\tau_i$. Then at least one frequency changes during an interval of $k' \leq k$ iterations with a probability of at most $k'n^{1 - \varepsilon/3}\log_2 k$.
\end{corollary}

\begin{proof}
	For any $i \in [n]$ during any of the $k$ iterations, by Lemma~\ref{lem:unbiasedFalsePositives}, the probability that $\tau_i$ changes is at most $n^{-\varepsilon/3}\log_2 k$. Via a union bound over all $k'$ relevant iterations and all $n$ frequencies, the statement follows.
\end{proof}

Intuitively, this corollary states that~$\varepsilon$ should be chosen such that the term $k'n^{1 - \varepsilon/3}\log_2 k$ represents the desired error probability, where~$k'$ is the length of an interval such that a frequency only drops with the error probability. Assuming that one chooses $k' = \Theta(n^r)$ for some constant $r > 0$ and desires an error probability of at most $n^{-q}$ for some constant $q > 0$ (ignoring constant factors and the logarithm), it makes sense to choose $\varepsilon \geq 3(r + 1 + q)$.

\subsection{Efficient Implementation of the \sigcGA}

Recall that, in order to save on the computational cost of checking for a significance, we only do so in historic data in lengths of powers of~$2$. By precomputing the number of $1$s in each such interval, checking a single history for a significance can be done in time logarithmically in its length. Note that the update to this precomputed data can also be done in logarithmic time, as each iteration only a single bit is added to the history and thus the number of~$1$s can only differ by at most one per interval from one iteration to the next. Consequently, the loop of the \sigcGA has a computational cost of $O(\sum_{i=1}^n \log |H_i|)$. Since a history can never be longer than the run time~$T$ of the \sigcGA, its total computational cost is $O(n T \log T)$. In comparison, many EAs have an extra cost of $O(n)$ per iteration. Thus, our significance-based approach is only more costly by a factor of $O(\log T)$.

One drawback of the approach above is that the full history needs to be stored. Thus, we describe a way of condensing a history to a size only logarithmic in the length of the full history. This approach does not allow anymore to access the exact number of $1$s (or $0$s) in all power-of-two length histories. However, for each $\ell \in [|H_i|]$, it yields the number of $1$s in some interval of length $\ell'$ with $\ell \le \ell' < 2 \ell$. For reasons of readability, we nevertheless regard the original \sigcGA in the subsequent analyses, but it is quite evident that the mildly different accessibility of the history in our condensed implementation does not change our result.

For our condensed storage of the history, we have a list of blocks, each storing the number of $1$s in some discrete interval $[t_1..t_2]$ of length equal to a power of two (including $1$). When a new item has to be stored, we append a block of size $1$ to the list. Then, traversing the list starting with the newest element, we check if there are three consecutive blocks of the same size, and if so, we merge the two oldest ones into a new block of twice the size. By this, we always maintain a list of blocks such that, for a certain power $2^k$, there are between one and two blocks of length $2^j$ for all $j \in [0 .. k - 1]$. This structural property implies both that we only have a logarithmic number of blocks (as we have $k = O(\log |H_i|)$) and that we can (in amortized constant time) access all historic intervals consisting of full blocks, which in particular implies that we can access an interval with length in $[2^j,2^{j+1}-1]$ for all $j \in [0 .. k]$. 

For the pseudocode of this approach, assume that a list element has a pointer to the next list element (next), its previous element (prev), and stores an integer value (load).

\begin{algorithm2e}
	\caption{The algorithm used by the \sigcGA in order to condense a history~$H$, given a new bit value~$x$. Let~$L$ be a list with elements of non-decreasing load.}
	\label{alg:condensedHistory}
	$X \gets \textrm{a new list element with load}~x$\;
	Append~$X$ to the head of~$L$\;
	$C \gets \textrm{head of}~L$\;
	$N \gets C.\mathrm{next}$\;
	$r \gets 0$\;
	\While{$N$ \emph{is not null}}
	{
		\lIf{$C.\mathrm{load} = N.\mathrm{load}$}{$r \gets r + 1$}
		\If{$r = 2$}
		{
			$r \gets 0$\;
			$X \gets \textrm{a new list element with load } C.\mathrm{load} +$ $N.\mathrm{load}$\;
			$C.\mathrm{prev}.\mathrm{next} \gets X$\;
			$X.\mathrm{next} \gets N.\mathrm{next}$\;
			$N \gets X$\;
		}
		$C \gets N$\;
		$N \gets N.\mathrm{next}$\;
	}
\end{algorithm2e}

\section{Run Time Results for \LO and \OM}
\label{sec:runTimeResultssigcGA}

We now prove our main results, that is, upper bounds of $O(n\log n)$ for the expected run time of the \sigcGA on \LO and \OM. We also note that the optimization process for the binary value function can be analyzed with arguments very similar to those for the \LO process. Consequently, we here have an $O(n \log n)$ run time as well.
Further, we consider the number of iterations~$T$ until the \sigcGA finds the optimal solution. Since it generates two offspring each iteration, the number of fitness function evaluations is at most $2T$.

Note that the \sigcGA treats~$1$s and~$0$s symmetrically, that is, it is unbiased in the sense of Lehre and Witt~\cite{DBLP:journals/algorithmica/LehreW12}. Hence, all results in this section hold for any type of function as defined in eqs.~\eqref{eq:oneMax} or~\eqref{eq:leadingOnes} where, for any position $i \in [n]$, a bit~$x_i$ can be flipped to $1 - x_i$ instead or swapped with another bit~$x_j$.

The following lemma states a useful bound for convex combinations. We use it in order to bound the probability of an event that we decomposed into an event and its complement.

\begin{lemma}
	\label{lem:convexCombination}
	Let $\alpha, \beta, x, y \in \R$ such that $x \leq y$ and $\alpha \leq \beta$. Then $\alpha x + (1 - \alpha)y \geq \beta x + (1 - \beta)y$.
\end{lemma}

\subsection{Analysis of \LO}

We show that the frequencies are set to $1 - 1/n$ sequentially from the most significant bit position to the least significant, that is, from left to right. w.h.p., no frequency is decreased until the optimization process is finished. Thus, a frequency~$\tau_i$ will stay at~$1/2$ until all of the frequencies to its left are set to $1 - 1/n$. Then~$\tau_i$ will become relevant for selection, as all of the frequencies left to it will only sample~$1$s w.h.p. This results in a significant surplus of~$1$s being saved at position~$i$, and~$\tau_i$ will be set to $1 - 1/n$ within $O(\log n)$ iterations and remain there. Then frequency~$\tau_{i + 1}$ becomes relevant for selection. As we need to set~$n$ frequencies to $1 - 1/n$, we get a run time of $O(n \log n)$.

\begin{theorem}
	\label{thm:leadingOnes}
	Consider the \sigcGA (Alg.~\ref{alg:adaptiveEDA}) with $\varepsilon > 12$ being a constant. Its run time on \LO is $O(n \log n)$ w.h.p. and in expectation.
\end{theorem}
\begin{proof}
	We split this proof into two parts and start by showing that the run time is $O(n \log n)$ w.h.p. Then we prove the expected run time.
	
	\textbf{Run time w.h.p.} For the first part of the proof, we consider the first $O(n \log n)$ iterations of the \sigcGA and condition on the event that no frequency decreases during this time, i.e., no (false) significance of $0$s is detected. Since, for any position $i \in [n]$ in \LO, having a~$1$ is always at least as good as having a~$0$, a~$1$ is saved in~$H_i$ with a probability of at least $\tau_i$. Hence, by Corollary~\ref{cor:noFrequenciesDrop}, no frequency decreases in the first $O(n \log n)$ iterations with a probability of at least $1 - O(n^{2 - \varepsilon/3}\log^2 n)$. As $\varepsilon > 12$, for an $\varepsilon' > 2$, this probability is at least $1 - O(n^{-\varepsilon'})$, which is w.h.p. In the following, we condition on this event.
	
	We now prove that the history of the leftmost position with a frequency at $1/2$ saves~$1$s significantly more often than~$0$s such that the frequency is set to $1 - 1/n$ after $O(\log n)$ iterations.
	For the second part of the proof, we use a similar argument, but the frequency is at $1/n$, and it takes $O(n\log n)$ steps to get to $1 - 1/n$. Since the calculations for both scenarios are very similar, we combine them.%
	
	Consider a position $i \in [n]$ and any of the first $O(n \log n)$ iterations such that $\tau_i \in \{1/n, 1/2\}$ and, for all positions $j < i$, $\tau_j = 1 - 1/n$. Let $O$ denote the event that we save a~$1$ in~$H_i$ this iteration. We derive an upper bound on the probability to detect the significance of~$1$s in $H_i$ within $O(\log n)$ iterations by calculating a lower bound on the probability of~$O$.
	Note that the probability of~$O$ is the same for each iteration until~$\tau_i$ is increased, since we condition on no frequency dropping within the first $O(n \log n)$ iterations.
	
	In order to bound $\Pr[O]$, we consider the event~$A$ that the bit at position $i$ of the winning individual is not relevant for selection. That is, $A$ denotes the event that at least one of the two offspring during this iteration has a $0$ at a position in $[i - 1]$. Thus, if~$A$ occurs, a~$1$ is saved with probability $\tau_i \in \{1/n, 1/2\}$. Otherwise, a~$1$ is saved if not two~$0$s are sampled, which has a probability of $1 - (1 - \tau_i)^2$. Hence,
	\[
	\Pr[O] = \Pr[A] \cdot \tau_i + \Pr\!\big[\overline{A}\big] \cdot \big(1 - (1 - \tau_i)^2\big)\ ,
	\]
	which is a convex combination of $\tau_i$ and $1 - (1 - \tau_i)^2$.
	By Lemma~\ref{lem:convexCombination}, decreasing $\Pr\!\big[\overline{A}\big]$ (as $1 - (1 - \tau_i)^2 = \tau_i(2 - \tau_i) \geq \tau_i$) results in a lower bound of $\Pr[O]$.
	Since $\overline{A}$ is equivalent to both offspring having only~$1$s at positions in $[i - 1]$, we see that
	\[
	\Pr\!\big[\overline{A}\big] = \left(1 - \tfrac{1}{n}\right)^{2(i - 1)}\ ,
	\]
	due to our assumption that all frequencies left of position~$i$ are at $1 - 1/n$.
	As this term is minimal for $i = n$, using the well-known inequality $(1 - 1/n)^{n - 1} \geq e^{-1}$, we bound $\Pr\!\big[\overline{A}\big] \geq e^{-2}$. Further, noting that $1 - (1 - \tau_i)^2 \geq (3/2)\tau_i$ for $\tau_i \in \{1/n, 1/2\}$, we bound
	\begin{align*}
		\Pr[O] &\geq (1 - e^{-2}) \cdot \tau_i + e^{-2} \cdot \big(1 - (1 - \tau_i)^2\big)\\
		&\geq (1 - e^{-2}) \cdot \tau_i + \tfrac{3}{2}e^{-2}\tau_i = \left(1 + \tfrac{1}{2}e^{-2}\right) \cdot \tau_i\ .
	\end{align*}
	
	Given our bound on the probability of~$O$, we now bound the probability to detect a significance of~$1$s in~$H_i$ within~$k$ iterations. To this end, let $X \sim \mathrm{Bin}\big(k, (1 + e^{-2}/2)\tau_i\big)$, and note that~$X$ is stochastically dominated by the process of saving $1$s at position~$i$. We bound the probability that we do not detect a significance of~$1$s within~$k$ iterations:
	\begin{align*}
		&\Pr\left[X < k\tau_i + s\left(\varepsilon, k\tau_i\right)\right]\\
		&\hspace*{5 em}\leq \Pr\left[X \leq E[X] - \left(\tfrac{k}{2}e^{-2}\tau_i - s\left(\varepsilon, k\tau_i\right)\right)\right]\ .
	\end{align*}
	Note that the minuend is positive for $k > (4/\tau_i)e^4 \varepsilon^2 \ln n > \ln n$, which holds due to our assumption $\varepsilon > 12$.
	Let $c = (4/\tau_i)e^4 \varepsilon^2$, and assume $k \geq 8c\ln n$. Thus, $(k/2)e^{-2}\tau_i - s(\varepsilon, k\tau_i) \geq (k/4)e^{-2}\tau_i \eqqcolon \lambda$ and $\Var[X] = k (1 + e^{-2}/2)\tau_i \big(1 - (1 + e^{-2}/2)\tau_i\big) \geq \lambda$. By Theorem~\ref{thm:additiveChernoff}, noting that $\lambda^2/\Var[X] \leq \lambda$ and using $\Var[X] \leq 2k\tau_i$, we bound
	\begin{align*}
		&\Pr\left[X < k\tau_i + s\left(\varepsilon, k\tau_i\right)\right] \leq \Pr\left[X \leq E[X] - \tfrac{k}{4}e^{-2}\tau_i\right]\\
		&\leq e^{-\frac{1}{3}\cdot\frac{\lambda^2}{\Var[X]}} \leq e^{-\frac{1}{3}\cdot\frac{k^2e^{-4}\tau_i^2}{16 \cdot 2k\tau_i}} = e^{-\frac{1}{3}\cdot\frac{ke^{-4}\tau_i}{32}}\\
		&\leq n^{-\frac{1}{3}\cdot\frac{ce^{-4}\tau_i}{4}} = n^{-\frac{\varepsilon^2}{3}}\ .
	\end{align*}
	Hence, $\tau_i$ is set to $1 - 1/n$ after $(4/\tau_i)e^4\varepsilon^2\ln n = O\big((1/\tau_i)\log n\big)$ iterations with a probability of at least $1 - n^{-\varepsilon^2/3}$. By applying a union bound over all~$n$ different possibilities for index~$i$, we see that each frequency (once all frequencies at positions $[i - 1]$ are at $1 - 1/n$) is set to $1 - 1/n$ within $O\big((1/\tau_i)\log n\big)$ with a probability of at least $1 - n^{1 - \varepsilon^2/3} \geq 1 - n^{-47}$, since $\varepsilon > 12$, which is w.h.p.
	
	Overall, we assume that no frequency decreases during the first $O(n \log n)$ iterations w.h.p., and we showed that each frequency (at~$1/2$) is set to $1 - 1/n$ within $O(\log n)$ iterations w.h.p. once all frequencies to its left are at $1 - 1/n$. Thus, since there are~$n$ frequencies, all frequencies are at $1 - 1/n$ after $O(n \log n)$ iterations w.h.p. The probability to sample the optimum is now $(1 - 1/n)^n \geq 1/(2e) = \Omega(1)$. Hence, waiting an additional $O(\log n)$ iterations, the optimum is sampled w.h.p. This concludes the first part of this proof.
	
	\textbf{Expected run time.} Since we showed above that the \sigcGA optimizes \LO in $O(n \log n)$ iterations w.h.p., we are left to bound its run time in the event that at least one frequency decreases within the first $O(n \log n)$ iterations. As we discussed at the beginning of the first part of this proof, this only happens with a probability of $O(n^{-\varepsilon'})$, for $\varepsilon' > 2$.
	
	Consider an interval of length~$t'$. By Corollary~\ref{cor:noFrequenciesDrop}, during the first~$t$ iterations, no frequency decreases for~$t'$ iterations with a probability of at least $1 - t'n^{1 - \varepsilon/3}\log_2 t$. Assume $t \leq n^{2n}$ and $t' = \Theta(n^2\log n)$. Then no frequency decreases for~$t'$ iterations w.h.p., since $\varepsilon > 12$.
	
	By using the result calculated in the first part, we see that a leftmost frequency~$\tau_i$ at~$1/n$ is increased during $O\big((1/\tau_i)\log n\big) = O(n\log n)$ iterations w.h.p. Thus, in overall, the \sigcGA finds the optimum during an interval of length $t' = \Theta(n^2\log n)$ w.h.p., as~$n$ frequencies need to be increased to $1 - 1/n$. We pessimistically assume that the optimum is only found with a probability of at least $1/2$ during~$t'$ iterations. Hence, the expected run time in this case is $2t' = \Theta(t')$.
	
	Last, we assume that we did not find the optimum during $n^{2n}$ iterations, which only happens with a probability of at most $2^{-n^{2n}/t'}$. Then, the expected run time is at most~$n^n$ by pessimistically assuming that all frequencies are at $1/n$.
	
	We conclude the proof by combining all of the three different regimes we just discussed, we see that we can bound the expected run time by
	\[
	O(n \log n) + O(n^{-\varepsilon'})\cdot O(t') + 2^{-n^{2n}/t'} \cdot n^n = O(n \log n)\ .\qedhere
	\]
\end{proof}

The proof of Theorem~\ref{thm:leadingOnes} shows that the \sigcGA rapidly makes progress when optimizing~\LO. In fact, after $O(i\log n)$ iterations, with $i \in [n]$, the \sigcGA finds a solution with fitness~$i$ w.h.p. (if $i$ is large) and in expectation. Thus, in the fixed-budget perspective introduced by Jansen and Zarges~\cite{JansenZ14}, the \sigcGA performs very well on~\LO. For comparison, for the $(1 + 1)$~EA, it is known that the time to reach a fitness of $i$ is $\Theta(in)$ in expectation and (again, when $i$ is sufficiently large) w.h.p., see~\cite{DoerrJWZ13}.

The reason that the~\sigcGA optimizes~\LO so quickly is that the probability of saving a~$1$ at position~$i$ is increased by a constant factor once all frequencies at positions less than~$i$ are at $1 - 1/n$. This boost is a result of position~$i$ being the most relevant position for selection, assuming that all bits at positions less than~$i$ are~$1$.

\paragraph*{Binary Value} A very similar boost in relevance occurs when considering the function~\BinVal ($\BV$), which returns the bit value of a bit string. Formally, $\BV$ is defined~as
\[
\BV(x) = \sum_{i = 1}^{n} 2^{n - i}x_i\ .
\]
Note that the most significant bit is the leftmost.

\BV imposes a lexicographic order from left to right on a bit string~$x$, since a bit~$x_i$ has a greater weight than the sum of all weights at positions greater than~$i$. This is similar to \LO. The main difference is that, for \BV, a position~$i$ can also be relevant for selection when bits at positions less than~$i$ are~$0$. More formally, for \LO, position~$i$ is only relevant for selection when all of the bits at positions less than~$i$ are~$1$, whereas position~$i$ is relevant for selection for \BV when all the bits at positions less than~$i$ are \emph{the same.} With this insight, we adapt the proof of Theorem~\ref{thm:leadingOnes} for \BV. 

\begin{corollary}
	\label{cor:binVal}
	Consider the \sigcGA (Alg.~\ref{alg:adaptiveEDA}) with $\varepsilon > 12$ being a constant. Its run time on \BV is $O(n \log n)$ w.h.p. and in expectation.
\end{corollary}

\subsection{Analysis of \OM}

In order to analyze how likely it is that two individuals sampled from the \sigcGA have the same \OM value, we use the following estimate, whose proof can be found, e.g., in~\cite{DoerrW14ranking}. 
\begin{lemma}
	\label{lem:centralBinomialCoefficient}
	For $c \in \Theta(1)$, $\ell \in \N^+$, let $k \in [\ell/2 \pm c\sqrt{\ell}]$ and let $X \sim \mathrm{Bin}(1/2, \ell)$. Then $\Pr[X = k] = \Omega\left(\tfrac{1}{\sqrt{\ell}}\right)$.
\end{lemma}

For the proof of the run time of the \sigcGA on \OM, we show that, during \emph{each} of the first $O(n\log n)$ iterations, each position can become relevant for selection with a decent probability of $\Omega(1/\sqrt{n})$. In contrast to \LO, there is no sudden change in the probability that~$1$s are saved. Thus, it takes $O(n\log n)$ iterations to set a frequency to $1 - 1/n$. However, this is done for all frequencies in parallel. Thus, the overall run time remains $O(n\log n)$.

\begin{theorem}
	\label{thm:oneMax}
	Consider the \sigcGA (Alg.~\ref{alg:adaptiveEDA}) with $\varepsilon > 12$ being a constant. Its run time on \OM is $O(n \log n)$ w.h.p. and in expectation.
\end{theorem}
\begin{proof}
	We first show that the run time holds w.h.p. Then we prove the expected run time.
	
	\textbf{Run time w.h.p.} We consider the first $O(n \log n)$ iterations and condition on the event that no frequency decreases during that time. This can be argued in the same way as at the beginning in the proof of Theorem~\ref{thm:leadingOnes}.
	
	We now show that a single frequency (starting at~$1/2$) is increased to $1 - 1/n$ within the first $O(n \log n)$ iterations w.h.p. as long as the other frequencies are at~$1/2$ or at $1 - 1/n$. Hence, \emph{all} frequencies are increased during that time w.h.p. when applying a union bound.
	
	Similar to the proof of Theorem~\ref{thm:leadingOnes}, when proving the expected run time, we use that, if all frequencies start at~$1/n$, they are set to $1 - 1/n$ w.h.p. within $O(n^2 \log n)$ iterations in parallel. Thus, we combine both cases in the following argumentation.
	
	Let $s \in \{1/2, 1/n\}$ denote the starting value of a frequency that we consider, and let $\ell \in [n]$ denote the number of frequencies \emph{not} at $1 - 1/n$ during an arbitrary single iteration. Further, let $i \in [n]$ be a position in that iteration such that $\tau_i = s$. We prove that~$H_i$ saves~$1$s more likely by a factor of $1 + \Theta(1/\sqrt{\ell})$ when compared to the hypothesis. This results in~$\tau_i$ being increased to $1 - 1/n$ within $O\big((\ell/s) \log n\big)$ iterations.
	
	We determine the bias in saving a~$1$ by making the following observation: ignoring position~$i$, if the absolute difference in the number of~$1$s of both offspring is greater than one, then bit~$i$ is not relevant for determining which offspring is selected. However, if the difference in the number of~$1$s (except position~$i$) of both offspring is at most~$1$, having a~$1$ at position~$i$ makes it more likely for an individual to be selected. We now formalize this idea.
	To this end, let~$O$ denote the event that~$H_i$ saves a~$1$, and let~$A$ denote the event that the difference of both offspring (except position~$i$) is greater than~$1$. Note that in the case of~$A$, the probability to save a~$1$ is~$\tau_i$.
	
	We now consider the case of~$\overline{A}$, that is, the absolute difference in the number of~$1$s of both offspring (excluding position~$i$) is at most one. If it is zero, then~$H_i$ saves a~$1$ if none of the offspring has a~$0$ at position~$i$. Thus, the respective probability is $1 - (1 - \tau_i)^2 = 2\tau_i - \tau_i^2$. In the case of the numbers of~$1$ differing by exactly one, a~$1$ is saved if the winner (with respect to all bits but bit $i$) has a~$1$ at position~$i$ (which it has with a probability of~$\tau_i$), or if the winner has a~$0$ at position~$i$, the loser has a~$1$, and the loser wins the tie-breaking. The probability of this event is $(1/2)\tau_i(1 - \tau_i) \geq (1/4)\tau_i$. All in all, the probability to save a~$1$ conditional on~$\overline{A}$ is at least $\tau_i + (1/4)\tau_i = (5/4)\tau_i$, since $(5/4)\tau_i \leq 2\tau_i - \tau_i^2$ for $\tau_i \in \{1/2, 1/n\}$.
	
	Taking both cases together, we bound
	\begin{align*}
		\Pr[O] \geq \Pr[A]\cdot\tau_i + \Pr\!\big[\overline{A}\big]\cdot\frac{5}{4}\tau_i\ .
	\end{align*}
	By Lemma~\ref{lem:convexCombination}, we lower bound this term even further by calculating a lower bound for $\Pr\!\big[\overline{A}\big]$.
	We first show that the frequencies at $1 - 1/n$ and~$1/n$ sample the same bits in both offspring with at least a constant probability. For the $n - \ell$ positions with frequencies at $1 - 1/n$, both offspring have a~$1$ at the respective positions with a probability of $(1 - 1/n)^{2(n - \ell)} \geq e^{-2}$, since $n - \ell \leq n - 1$.  Analogously, for all positions with frequencies at~$1/n$ (but~$\tau_i$), both offspring have a~$0$ at position~$i$ also with a probability of at least $(1 - 1/n)^{2(n - 1)} \geq e^{-2}$. Hence, both offspring have the same bits at all positions with frequencies not at~$1/2$ with a probability of at least~$e^{-4}$.
	
	We now consider the number of~$1$s of an offspring at the remaining $\ell' \leq \ell - 1$ (for $\ell \geq 2$) positions (except~$i$) with frequencies at~$1/2$. We call this number~$Y$. Note that the expected value of~$Y$ is $\ell'/2$. By Lemma~\ref{lem:centralBinomialCoefficient}, for a $k \in [\ell'/2 \pm \sqrt{\ell'/2}]$, the probability that $Y = k$ is $\Omega(1/\sqrt{\ell'})$. Thus, the probability that both offspring have the same number of~$1$s at the~$\ell'$ positions we consider is $d/\sqrt{\ell'}$, for a constant $d > 0$, since there are $\sqrt{\ell'}$ possible values of~$k$ and the probability that both offspring have~$k$ bits as~$1$ is $\Omega(1/\ell')$. Factoring in the probability of all remaining $n - \ell'$ positions to sample the same values in both offspring and for a sufficiently small constant $d' > 0$, we bound
	\begin{align*}
		\Pr[O] &\geq \left(1 - e^{-4}\frac{d}{\sqrt{\ell'}}\right)\cdot\tau_i + e^{-4}\frac{d}{\sqrt{\ell'}}\cdot\frac{5}{4}\tau_i\\
		&\geq \left(1 + \frac{d'}{\sqrt{\ell}}\right)\tau_i\ .
	\end{align*}
	Recall that we assumed $\ell \geq 2$ for this bound. For $\ell = 1$, i.e., $\ell' = 0$, we have $n - 1$ positions with frequencies at $1 - 1/n$ or $1/n$. Thus, $\Pr\!\big[\overline{A}\big] \geq e^{-4}$, as we discussed before. Consequently, we bound $\Pr[O] \geq (1 - e^{-2})\cdot \tau_i + e^{-2}\cdot (5/4)\tau_i \geq (1 + d'/\sqrt{\ell})\tau_i$ if we choose~$d'$ sufficiently small. Overall, we use $(1 + d'/\sqrt{\ell})\tau_i$ as a lower bound for $\Pr[O]$.
	
	Analogous to the proof of Theorem~\ref{thm:leadingOnes}, we now consider the probability to detect a significance of~$1$s in~$H_i$ within~$k$ iterations. To this end, let $X \sim \mathrm{Bin}\big(k, (1 + d'/\sqrt{\ell})\tau_i\big)$ and note that~$X$ is stochastically dominated by the process of saving~$1$s at position~$i$. We bound the probability to not detect a significance of~$1$s as follows:
	\begin{align*}
		&\Pr\left[X < k\tau_i + s\left(\varepsilon, k\tau_i\right)\right]\\
		&\hspace*{6 em}\leq \Pr\left[X \leq E[X] - \Bigg(\frac{kd'}{\sqrt{\ell}}\tau_i - s\left(\varepsilon, k\tau_i\right)\Bigg)\right]\ .
	\end{align*}
	Let $k \geq 4(\varepsilon^2/d'^2)(\ell/\tau_i)\ln n$. Then $(kd'/\sqrt{\ell})\tau_i - s(\varepsilon, k\tau_i) \geq \big(kd'/(2\sqrt{\ell})\big)\tau_i \eqqcolon \lambda$. Further note that $\Var[X] = k\tau_i(1 - \tau_i) \geq \lambda$ if~$d'$ is sufficiently small, which implies $\lambda^2/\Var[X] \leq \lambda$. By Theorem~\ref{thm:additiveChernoff} with $\Var[X] \leq k\tau_i$, we see that
	\begin{align*}
		\Pr\left[X < k\tau_i + s\left(\varepsilon, k\tau_i\right)\right] &\leq \Pr\left[X \leq E[X] - \frac{kd'}{2\sqrt{\ell}}\tau_i\right]\\
		&\hspace*{-7 em}\leq e^{-\frac{1}{3}\cdot\frac{4k^2d'^2\tau_i^2}{4\ell k\tau_i}} = e^{-\frac{1}{3}\cdot\frac{kd'^2}{\ell}\tau_i} \leq e^{-\frac{4}{3}\varepsilon^2\ln n} = n^{-\frac{4}{3}\varepsilon^2}\ .
	\end{align*}
	Hence, $\tau_i$ is increased to $1 - 1/n$ within $4(\varepsilon^2/d'^2)(\ell/\tau_i)\ln n = O\big((\ell/\tau_i)\log n\big)$ iterations with a probability of at least $1 - n^{-4\varepsilon^2/3}$. By applying a union bound over all~$n$ frequencies, each frequency reaches $1 - 1/n$ within $O\big((\ell/\tau_i)\log n\big)$ iterations with a probability of at least $1 - n^{1 - 4\varepsilon^2/3} \geq 1 - n^{-191}$, as $\varepsilon > 12$, which is w.h.p.
	
	Since we assume that no frequency drops within the first $O(n \log n)$ iterations w.h.p. and since all frequencies start at~$1/2$, all of them reach $1 - 1/n$ within that time w.h.p. Then, the optimum is sampled during a single iteration with a probability of at least $(1 - 1/n)^n \geq 1/(2e) = \Omega(1)$. Thus, the optimum is sampled after $O(\log n)$ additional iterations w.h.p.
	
	\textbf{Expected run time.} This part follows the same arguments as outlined in the respective part in the proof of Theorem~\ref{thm:leadingOnes}. Different from there, assuming that a frequency is at~$1/n$, it now takes $O(n^2\log n)$ iterations to be increased to $1 - 1/n$ w.h.p., as we proved above. However, since $\varepsilon > 12$, a union bound over all~$n$ frequency again results in all frequencies being increased during $O(n^2\log n)$ iterations w.h.p. The rest remains the same, which concludes this proof.
\end{proof}

\textbf{\OM vs. \LO.} While the \sigcGA has the same asymptotic run time on \LO and \OM (w.h.p. and in expectation), the reasons differ. For \LO, the frequencies are increased consecutively to $1 - 1/n$, where each frequency only needs $O(\log n)$ iterations, which is also the asymptotic minimum number of iterations to do so. This speed results from the sudden boost in probability once a position~$i$ becomes relevant, that is, its preceding frequencies are all at $1 - 1/n$ and thus sample only~$1$s with at least a constant probability. Given that both offspring have only~$1$s at positions in $[i - 1]$, it suffices that position~$i$ has at least one~$1$, which is quite likely. In contrast, the probability that the bias in selection is also detected at positions after~$i$ declines exponentially in the distance to~$i$, making the bias negligible. This fact is also exploited by Friedrich et~al.~\cite{DBLP:conf/gecco/FriedrichKK16} in the analysis (and design) of the scGA, which is why it has the same run time on \LO.

For \OM, the impact of the bias in selection depends on the number~$\ell$ of other frequencies that are not at $1 - 1/n$. In order for a position~$i$ to detect the bias, the number of~$1$s in these positions has to almost be identical in both offspring, i.e., it can differ by at most one. This then adds a bias of roughly $1/\sqrt{\ell}$ for saving a~$1$ in~$H_i$. Since~$\ell$ is large for a long time (for example, $\ell = n$ at the beginning), this bias remains small during that period. However, this bias is there constantly for each position. Thus, all frequencies can be optimized in parallel, whereas for \LO this is done sequentially.

\textbf{\OM vs. \BV.} \BV is often considered one extremal case of the class of linear functions, as its weights impose a lexicographic order on the bit positions. The other extreme is \OM, where all weights are identical and basically no order among the positions exists. Our results show that the \sigcGA optimizes both functions in $O(n \log n)$. It remains an open question whether the \sigcGA is capable of optimizing any linear function in that time, a feat that the $(1 + 1)$~EA, a classical EA, is known to be capable of~\cite{DBLP:journals/tcs/DrosteJW02}. Contrary to that, it was proven for the cGA (an EDA) that it performs worse on \BV than on \OM~\cite{DBLP:journals/nc/Droste06}. Thus, a uniform performance on the class of linear functions would be a great feat for an EDA.

We would like to note that the result of Droste~\cite{DBLP:journals/nc/Droste06} considered the cGA without frequency borders, that is, the frequencies could reach values of~$0$. Once this is the case, the algorithm is stuck (as it only samples~$0$ at this position) and the optimization fails. It is unknown up to date whether the cGA still performs worse on \BV when the frequencies are bound to the interval $[1/n, 1 - 1/n]$. However, the main idea of Droste's proof that frequencies drop very low remains. Thus, if sufficiently many frequencies were to drop to~$1/n$, the cGA would still perform badly on \BV. Note that this is exactly the problem that the \sigcGA circumvents with its update rule, resulting in its run time of $O(n \log n)$.

The only other known EDA run time result for \BV was recently proven by Lehre and Nguyen~\cite{LehreN18PBIL}. They show that the PBIL optimizes \BV with $O(n^2)$ fitness function evaluations in expectation (considering best parameter choices).

\section{Run Time Analysis for the scGA}
\label{sec:scGA}

Another variant of the cGA~\cite{HarikLG98} that is able to optimize \LO in $O(n\log n)$ w.h.p. is the \emph{stable compact genetic algorithm} (scGA; Alg.~\ref{alg:scGA}) introduced by Friedrich et~al.~\cite{DBLP:conf/gecco/FriedrichKK16} with the intent to provide an EDA that optimizes \LO in $o(n^2)$. The update procedure of the scGA is very similar to that of the cGA, that is, a frequency at position~$i$ is changed by a value of~$\rho$ with respect to the difference of the bits at~$i$ of the winner and the loser. However, an update toward~$1/2$ is stronger by an additive term of~$a$, where $a \in O(\rho)$ is an additional parameter.

Different from many other EDAs, the scGA does not have a margin and explicitly makes use of the frequency values~$0$ and~$1$. In fact, the scGA has another parameter $d \in (1/2, 1)$, which indicates a value that is sufficient in order to set a frequency to~$1$. The value $1 - d$ is used symmetrically in order to set a frequency to~$0$. Thus, the scGA fixes frequencies once they leave the interval $(1 - d, d)$.

\begin{theorem}
	\label{thm:scGAOneMax}
	Let $\alpha \in (0, 1]$ be a constant. Consider the scGA with $\rho = O(1/\log n)$, $a = \alpha\rho$, and $1/2 < d \leq 5/6$ with $d = \Theta(1)$. Its run time on \OM is $\Omega\big(\!\min\{2^{\Theta(n)}, 2^{c/\rho}\}\big)$ in expectation and w.h.p. for a constant $c > 0$.
\end{theorem}

\begin{algorithm2e}
	\caption{The scGA~\cite{DBLP:conf/gecco/FriedrichKK16} with parameters $\rho$, $a$, and $d$ optimizing $f$}
	\label{alg:scGA}
	$t \gets 0$\;
	\lFor{$i \in [n]$}{$\tau^{(t)}_i \gets \frac{1}{2}$}
	\Repeat{\emph{termination criterion met}}
	{
		$x, y \gets$ offspring sampled with respect to $\tau^{(t)}$\;
		$(x, y) \gets$ winner/loser of $x$ and $y$ with respect to $f$\;
		\For{$i \in [n]$}
		{
			\If{$x_i > y_i$}
			{
				\lIf{$\tau^{(t)}_i \leq \frac{1}{2}$}{$\tau^{(t + 1)}_i \gets \tau^{(t)}_i + \rho + a$}
				\lElseIf{$\frac{1}{2} < \tau^{(t)}_i < d$}{$\tau^{(t + 1)} \gets \tau^{(t)}_i + \rho$}
				\lElse{$\tau^{(t + 1)}_i \gets 1$}
			}
			\ElseIf{$x_i < y_i$}
			{
				\lIf{$\tau^{(t)}_i \geq \frac{1}{2}$}{$\tau^{(t + 1)}_i \gets \tau^{(t)}_i - \rho - a$}
				\lElseIf{$1 - d < \tau^{(t)}_i < \frac{1}{2}$}{$\tau^{(t + 1)}_i \gets \tau^{(t)}_i - \rho$}
				\lElse{$\tau^{(t + 1)}_i \gets 0$}
			}
			\lElse{$\tau^{(t + 1)}_i \gets \tau^{(t)}_i$}
		}
		$t \gets t+1$\;
	}
\end{algorithm2e}

Before we prove the theorem, we mention two other theorems that we are going to use in the proof. The first bounds the probability of a randomly sampled bit string having $s \in \{0\} \cup [n]$ $1$s. We use it in order to bound the probability of both offspring having the same number of $1$s. Note that the values $1/6$ and $5/6$ in the lemma are somewhat arbitrary and can be exchanged for any constant in $(0, 1/2)$ and $(1/2, 1)$, respectively.
\begin{lemma}[\cite{SudholtW19cGAACOOneMaxLowerBound}]
	\label{lem:independentPoissonSum}
	Let $S$ denote the sum of $n$ independent Poisson trials with probabilities $\tau_1, \ldots, \tau_n$ such that, for all $i \in [n]$, $1/6 \leq \tau_i \leq 5/6$. Then, for all $s \in \{0\} \cup [n]$,
	\[
	\Pr[S = s] = O\left(\frac{1}{\sqrt{n}}\right)\ .
	\]
\end{lemma}

The next theorem provides an upper bound on the probability of a random process stopping after a certain time. We use it in order to show that it is unlikely for a frequency of the scGA when optimizing \OM to get to $1$ within a certain number of iterations.
\begin{theorem}[Negative Drift; \cite{DBLP:journals/algorithmica/OlivetoW11, DBLP:journals/corr/OlivetoW12}]
	\label{thm:negDrift}
	Let $(X_t)_{t \in \N}$ be real-valued random variables describing a stochastic process over some state space, with $X_0 \geq b$. Suppose there exist an interval $[a, b] \subseteq \R$, two constants $\delta$, $\varepsilon > 0$, and, possibly depending on $\ell \coloneqq b - a$, a function $r(\ell)$ satisfying $1 \leq r(\ell) = o\big(\ell/\log(\ell)\big)$ such that, for all $t \in \N$, the following two conditions hold:
	\begin{enumerate}
		\item\label{item:negativeDrift} $E[X_{t + 1} - X_t \mid X_t \land a < X_t < b] \geq \varepsilon$ and,\\
		
		\item\label{item:stepWidth} for all $j \in \N$, $\Pr[|X_{t + 1} - X_t| \geq j \mid X_t \land X_t > a] \leq \frac{r(\ell)}{(1 + \delta)^j}$.
	\end{enumerate}
	Then there is a constant $c > 0$ such that, for $T \coloneqq \min\{t \in \N \mid X_t \leq a\}$, it holds that
	\[
	\Pr\left[T \leq 2^{\frac{c\ell}{r(\ell)}}\right] = 2^{-\Omega\left(\frac{\ell}{r(\ell)}\right)}\ .
	\]
\end{theorem}

We now prove our result.

\begin{proof}[Proof of Theorem~\ref{thm:scGAOneMax}]
	We only show that the run time is in $\Omega\big(\!\min\{2^{\Theta(n)}, 2^{c/\rho}\}\big)$ w.h.p. The statement for the expected run time follows by lower-bounding the terms that occur with a probability of $o(1)$ with~$0$.
	
	We first prove the bound of $\Omega\big(2^{c/\rho}\big)$.
	We do so by showing that each frequency will stay in the non-empty interval $(1 - d, d) \subset [1/6, 5/6]$ w.h.p. Although \OM introduces a bias into updating a frequency, it is too tiny in order to compensate the strong drift toward $1/2$ in the update.
	
	We lower-bound the expected time it takes the scGA to optimize \OM by upper-bounding the probability it takes a single frequency to leave the interval $(1 - d, d)$. Thus, we condition during the entire proof implicitly on the event that all frequencies are in the interval $(1 - d, d)$. Note that, in this scenario, the probability to sample the optimum during an iteration is at most $(5/6)^n$, which is exponentially small, even for a polynomial number of iterations.
	
	Consider an index $i \in [n]$ with $\tau_i \in (1 - d, d)$. We only upper-bound the probability it takes $\tau_i$ to reach $d$. Note that the probability of $\tau_i$ reaching $1 - d$ is at most that large, as \OM introduces a bias for $1$s into the selection process. Hence, we could argue optimistically for $\tau_i$ reaching $1 - d$ as we do for $\tau_i$ reaching $d$ by swapping $1$s for $0$s and considering $1 - \tau_i$ instead.
	
	Let $T$ denote the first point in time $t$ such that $\tau^{(t)}_i \geq d$. We want to apply Theorem~\ref{thm:negDrift} and show that it is unlikely for $\tau_i$ to reach $d$ within $2^{c/\rho}$ iterations. Hence, we define the following potential function $g\colon [0, 1] \to \R$:
	\[
	g(\tau_i) = \frac{1}{\rho}(1 - \tau_i)\ ,
	\]
	which we will use for our frequencies. Note that, at the beginning, $\tau_i$ is at $1/2$, i.e., $g(1/2) = 1/(2\rho)$. We stop once $\tau_i \geq d$, i.e., $g(\tau_i) \leq (1 - d)/\rho$. Thus, we consider an interval of length $\ell \coloneqq 1/(2\rho) - (1 - d)/\rho = (2d - 1)/\rho = \Theta(1/\rho)$, as $d > 1/2$ is a constant.

	We now argue how an update to $\tau_i$ is performed in order to estimate its expected value after an update, which is necessary in order to apply Theorem~\ref{thm:negDrift}. Consider, similar to the proof of Theorem~\ref{thm:oneMax}, that the bits of both offspring $x$ and $y$ for all positions but position $i$ have been determined. If the difference of the number of $1$s of both offspring is at least $2$, i.e., $\|x - y\|_1 - x_i + y_i \geq 2$, then the outcome of neither $x_i$ nor $y_i$ can change the outcome of the selection process. Thus, $\tau_i$ increases with probability $\tau_i(1 - \tau_i)$, as the winner offspring needs to sample a $1$ and the loser a $0$. Analogously, in this case, the probability that $\tau_i$ decreases is $\tau_i(1 - \tau_i)$, too.
	
	If the difference of the number of $1$s of both offspring is one, then, in order to increase $\tau_i$, the winner (with respect to all bits but bit $i$) needs to sample a $1$ and the loser a $0$, or the winner needs to sample a $0$, the loser a $1$, and the loser wins. The first case has a probability of $\tau_i(1 - \tau_i)$, the second of $(1/2)\tau_i(1 - \tau_i)$, due to the uniform selection when the offspring have equal fitness.
	In order to decrease $\tau_i$, the winner needs to sample a $0$, the loser a $1$, and the winner has to win, which has a probability of $(1/2)\tau_i(1 - \tau_i)$.
	
	If the difference of the number of $1$s of both offspring is zero, then $\tau_i$ is increased if any offspring samples a $1$ and the other samples a $0$. This has probability $2\tau_i(1 - \tau_i)$. In this case, it is not possible that $\tau_i$ is decreased.
	
	\newcommand*{\peqone}{p_{1}}
	\newcommand*{\peqzero}{p_{0}}
	In order to estimate the probabilities of when $\tau_i$ increases or decreases, we need to estimate the probabilities that the number of $1$s of both offspring differ by at least two, differ by exactly one, and differ by exactly zero. Let $\peqone$ denote the probability that this difference is one, and let $\peqzero$ denote the probability that the difference is zero. We now bound these probabilities.
	
	Assume that offspring $x$ has $k$ $1$s, where $k \in \{0\} \cup [n - 1]$, since we assume that bit $i$ has not been sampled yet. For $\peqzero$, $y$ needs to sample $k$ $1$s as well, and for $\peqone$, $y$ needs to sample $k - 1$ or $k + 1$ $1$s (such that the result is still in $\{0\} \cup [n - 1]$). Due to Lemma~\ref{lem:independentPoissonSum}, the probability for $y$ to have this many $1$s is $O(1/\sqrt{n})$, as we assume that all frequencies are in the interval $(1 - d, d) \subset [1/6, 5/6]$. Hence, by the law of total probability, we get
	\begin{align*}
		\peqzero = O\left(\tfrac{1}{\sqrt{n}}\right) \textrm{ and } \peqone = O \left(\tfrac{1}{\sqrt{n}}\right)\ .
	\end{align*}
	In the following, let $\gamma > 0$ be a constant such that $\peqzero \leq \gamma/\sqrt{n}$ and $\peqone \leq \gamma/\sqrt{n}$.
	
	We now consider the drift of $g(\tau_i)$ in any iteration $t$ such that $1/2 < \tau^{(t)}_i < d$, i.e., we show that condition~(\ref{item:negativeDrift}) of Theorem~\ref{thm:negDrift} holds. Let $\tau = \tau^{(t)}_i$ and $\tau' = \tau^{(t + 1)}_i$. Note that conditioning on $g(\tau)$ is the same as conditioning on $\tau$, as $g$ is injective. If $\tau$ increases, it changes by $\rho$, and if it decreases, it changes by $\rho + a$.
	\begin{align*}
		&E\left[g(\tau') - g(\tau) \ \bigg\vert\ g(\tau) \land \frac{1 - d}{\rho} < g(\tau) < \frac{1}{2\rho}\right]\\
		&= \tfrac{1}{\rho} E\left[\tau - \tau' \mid \tau \land \tfrac{1}{2} < \tau < d\right]\\
		&= \tfrac{1}{\rho}\Big((\rho + a)\left((1 - \peqzero - \peqone)\tau(1 - \tau) + \peqone\cdot\tfrac{1}{2}\tau(1 - \tau)\right)\\
		&\qquad -\rho\big((1 - \peqzero - \peqone)\tau(1 - \tau) + \peqone\cdot\tfrac{3}{2}\tau(1 - \tau)\\
		&\qquad\qquad +\ \peqzero\cdot 2\tau(1 - \tau)\big)\Big)\\
		&=\tfrac{1}{\rho}\tau(1 - \tau)\big((\rho + a)\left(1 - \peqzero - \tfrac{1}{2}\peqone\right) - \rho\left(1 + \peqzero + \tfrac{1}{2}\peqone\right)\big)\\
		&= \tfrac{1}{\rho}\tau(1 - \tau)\big(\rho(-2\peqzero - \peqone) + a\left(1 - \peqzero - \tfrac{1}{2}\peqone\right)\big)\ .
	\end{align*}
	For the negative terms with factor $a$, by using that $a \leq \rho$ and by applying the bounds on $\peqzero$ and $\peqone$, we get
	\begin{align*}
		&E\left[g(\tau') - g(\tau) \ \bigg\vert\ g(\tau) \land \frac{1 - d}{\rho} < g(\tau) < \frac{1}{2\rho}\right]\\
		&\geq \frac{1}{\rho}\tau(1 - \tau)\left(a - 3\peqzero\rho - \tfrac{3}{2}\peqone\rho\right)\\
		&\geq \frac{1}{\rho}\tau(1 - \tau)\left(a - 5\frac{\gamma}{\sqrt{n}}\rho\right)\ .
	\end{align*}
	Due to $a = \alpha\rho$, there is a sufficiently small constant $\beta > 0$ such that $a - 5\gamma\rho/\sqrt{n} \geq \beta\rho$. Thus, we get
	\begin{align*}
		E\left[g(\tau') - g(\tau) \ \bigg\vert\ g(\tau) \land \frac{1 - d}{\rho} < g(\tau) < \frac{1}{2\rho}\right] &\geq \beta\tau(1 - \tau)\\
		&\geq \beta\tfrac{1}{6}\cdot\tfrac{5}{6}\ ,
	\end{align*}
	which is constant.
	
	We now show that condition~(\ref{item:stepWidth}) of Theorem~\ref{thm:negDrift} holds. For this, we define $r(\ell) = 2$ and $\delta = \sqrt{2} - 1 > 0$. Note that $1 \leq r(\ell) = o\big(\ell/\log(\ell)\big) = o\big(1/(\rho\log(1/\rho))\big)$ holds, as $\rho = o(1)$. Since $\tau$ can change by at most $\rho + a \leq 2\rho$ during a single update, $g(\tau)$ can change by at most $2$. Thus, we only need to bound $\Pr[|g(\tau') - g(\tau)| \geq j \mid g(\tau) \land g(\tau) > (1 - d)\rho]$ for $j \in \{0, 1, 2\}$. For all of these three cases, $r(\ell)/\big((1 + \delta)^j\big) \geq 1$. Thus, condition~(\ref{item:stepWidth}) trivially holds for all $j \in \N$.
	
	Overall, by applying Theorem~\ref{thm:negDrift} and recalling that $\ell = \Theta(1/\rho)$, there are constant $c, c', c'' > 0$ such that
	\begin{align*}
		\Pr\left[T \leq 2^{\frac{c'\ell}{r(\ell)}}\right] = \Pr\left[T \leq 2^{\frac{c}{\rho}}\right] \leq 2^{-\Omega\big(\frac{1}{\rho}\big)} \leq n^{-c''}\ .
	\end{align*}
	Thus, w.h.p., $\tau_i$ does not reach $b$ within $2^{c/\rho}$ iterations, given that all frequencies are in $(1 - d, d)$. As discussed before, the probability of $\tau_i$ reaching $1 - d$ has at most the same probability. Note that conditioning on never sampling the optimum during any of these $t$~iterations increases these probabilities only by a factor of $1 - t(5/6)^n$, which is constant if $t = o(2^{\Theta(n)})$. Otherwise, we choose $2^{\Theta(n)}$ as run time bound. This concludes the proof.
\end{proof}

\section{Run Time Analysis for the Convex Search Algorithm}\label{sec:csa}

The following \emph{convex search algorithm} was proposed by Moraglio~\cite{DBLP:journals/ec/MoraglioS17}. Its sole parameter is a population size $\mu \in \N$. The algorithm starts with a first population of $\mu$ random individuals $x^{(1,1)}, \dots, x^{(1,\mu)} \in \{0,1\}^n$. In each iteration $t = 1, 2, \dots$, the algorithm generates from the current ``parent'' population $x^{(t,1)}, \dots, x^{(t,\mu)}$ a new ``offspring'' population $x^{(t+1,1)}, \dots, x^{(t+1,\mu)}$ as follows. 
\begin{itemize}
	\item If the parent population contains only copies of a single individual, the algorithm stops and outputs this solution. 
	\item If all individuals of the parent population have the same fitness, the offspring population is the parent population. 
	\item Otherwise, the individuals with lowest fitness value are removed from the parent population (giving the ``reduced parent population'') and the offspring population is obtained by $\mu$ times independently sampling from the convex hull of the reduced parent population. In other words, for all $i \in [n]$ and $j \in [\mu]$ independently, $x^{(t+1,j)}_i$ is chosen randomly from $\{0,1\}$ if the reduced parent populations contains both an individual having a $0$ at the $i$-th position and an individual having a $1$ at this position. If all individuals of the reduced parent population have the same value $b \in \{0,1\}$ in the $i$-th position, then $x^{(t+1,j)}_i \coloneqq b$.
\end{itemize}

The convex search algorithm with $\mu \ge 8 \log_2(4n^2 + n)$ and a suitable restart strategy was shown to optimize the \LO problem in expected time $O(n \log n)$~\cite{DBLP:journals/ec/MoraglioS17}. We now show that its performance on the \OM problem is not very attractive, namely it is asymptotically larger than any polynomial even when employing a suitable restart strategy. We suspect that much stronger lower bounds hold, but given the only moderate general interest in this algorithm so far, we restrict ourselves to this super-polynomial lower bound.

\begin{theorem}\label{thm:csa}
	Let $c > 0$. Regardless of the population size, a run of the convex search algorithm on \OM with probability at least $1 - O(n^{-c})$ 
	\begin{itemize}
		\item either reaches a state from which the optimum cannot be found,
		\item or within $n^c$ iterations does not fix any bit-position.
	\end{itemize}
	Consequently, at least $\Omega(n^c)$ iterations are necessary to find the optimum.
\end{theorem}

While the result may seem natural, proving it is made difficult by the dependencies inflicted from restricting the parent population to all but the lower fitness level. We remove these dependencies by suitable pessimistic estimates (e.g., estimating that a position does not become fixed to $1$ when at least $\Delta$ zeros are sampled), by suitable domination arguments (cf.~\cite{Doerr19StochasticDomination}), and by first regarding an artificial process in which bits can only be fixed to $1$.

\begin{proof}[Proof of Theorem~\ref{thm:csa}]
	Since we are aiming at an asymptotic statement, we assume in the following that $n$ is sufficiently large.
	
	To ease the following proof, we first argue that only the case $\mu = \Theta(\log n)$ is interesting. If $\mu \le \frac 12 \log_2 n$, then with probability 
	\[1 - (1 - 2^{-\mu})^n \ge 1 - \exp(-2^{-\mu} n) \ge 1 - \exp(-n^{1/2}),\] at least one of the bit-positions of the initial population is already converged to zero. Let now $\mu \ge K \log_2 n$ for a sufficiently large constant $K$ (which may depend on the constant $c$). We show that a random population with probability $n^{-2c}$ fixes no bit (that is, the next population is again fully random). By elementary properties of the binomial distribution, with probability at least $1 - (2/3)^{\mu}$, the lowest fitness value of the random population is below $n/2$. Since the probability of having a fitness of at least $n/2$ is at least $\frac 12$, the additive Chernoff bound (\cite[Theorem~10.7]{2018arXiv180106733D}) gives that with probability at least $1 - \exp(-\mu/8)$, the number $\mu^+$ of individuals having a fitness of at least $n/2$ is at least $\mu/4$. 
	
	We now condition on $\mu^+ \ge \mu/4$ and that there is an individual with fitness less than $n/2$ (and recall that this event happens with probability at least $1 - (2/3)^{\mu} - \exp(-\mu/8)$). We first note that in this case the $\mu^+$ individuals with fitness $n/2$ or more surely belong to the reduced population, which defines the next population. For each of these $\mu^+$ individuals and for each of their bit-positions, the probability to have a one is between $1/2$ and $3/4$, since these individuals are random individuals conditional on having at least $n/2$ ones. Since these individuals are stochastically independent, the probability that a bit-position in all these $\mu^+$ individuals has the value $0$ is at most $(1/2)^{\mu^+} \le (1/2)^{\mu/4}$ and the probability is at most $(3/4)^{\mu^+} \le (3/4)^{\mu/4}$ for the event that they are all one. 
	
	In summary, we obtain that the probability that a bit-position becomes fixed is at most $(2/3)^\mu + \exp(-\mu/8) + (1 - (2/3)^\mu - \exp(-\mu/8))n ((1/2)^{\mu/4}+(3/4)^{\mu/4})$. By taking the constant $K$ in the lower bound for $\mu$ sufficiently large, this probability is at most $n^{-2c}$. Hence a union bound over $n^c$ iterations shows that within this time frame, with probability $1 - n^{-c}$ no bit-position becomes fixed.
	
	In the remainder we thus assume that $\mu = \Theta(\log n)$ with implicit constants depending on the constant $c$ only. 
	
	We first regard the artificial random process which equals a true run of the CSA on \OM except that for all bits $i \in [n]$ where the reduced parent population contains only individuals with bit-value $0$ we still sample the offspring bits randomly from $\{0,1\}$. In other words, we prevent the algorithm from letting a bit-value converge to the wrong value of $0$. 
	
	Let $\Delta = \lceil 2 + 2c + 2\frac{\mu}{\log_2 n} \rceil$ and $t = \left\lfloor \frac{2^\mu}{4 (2\mu)^\Delta} \right\rfloor$.
	We call a bit-position $i \in [n]$ at time $s \in [t]$ \emph{unsafe} if the population at time $s$ contains at least $\mu - \Delta$ ones in this bit-position, that is, if $\sum_{j = 1}^\mu x^{(s,j)}_i \ge \mu - \Delta$, and if this bit-position was determined by sampling random bit-values (that is, not by setting all bit-values to one because the previous reduced population was converged to $1$ in this position). Let $X_{is}$ be the indicator random variable for this event. 
	
	We easily see that 
	\begin{align*}
	\Pr[X_{is}=1] = \Pr[\Bin(\mu,\tfrac 12) \ge \mu - \Delta] &\le 2^{-(\mu-\Delta)} \binom{\mu}{\mu-\Delta}\\
	&\le 2^{-\mu} (2\mu)^\Delta\ ,
	\end{align*}
	where the estimate for the binomial distribution is well-known (see~\cite[Lemma~3]{GiessenW17} or~\cite[Lemma~10.37]{2018arXiv180106733D}). 
	
	Regarding the correlation of the $X_{is}$, we see that either $X_{is}$ is a fresh random sample independent from all $X_{i's'}$ with $s' < s$ and $i' \in [n]$ or, namely if the reduced parent population in iteration $s$ has the $i$-th bit converged, $X_{is}=0$ with probability one. Consequently, the number $X \coloneqq \sum_{s=1}^t \sum_{i=1}^n X_{is}$ of unsafe bit-positions in the time frame $[t]$ is dominated by a sum of $nt$ independent Bernoulli random variables with success probability $2^{-\mu} (2\mu)^\Delta$, see~\cite[Lemma~11]{DoerrJ10} or~\cite[Lemma~10.22]{2018arXiv180106733D}.
	
	For these reasons, we have $E[X] \le \frac 14 n$ and $\Pr[X \ge \frac 12 n] \le \exp(-\frac 1 {8} n)$ by the additive Chernoff bound. 
	
	We now argue that having at least $\Delta$ zeros in some bit-position is often enough sufficient for the position not being fixed to $1$. For each $s \in [t]$, let $B_s$ be the following event. 
	\begin{itemize}
		\item If in the sampling process of the $s$-th population at least $\frac n2$ bit-positions are not already fixed (``fat $s$-th population''), then $B_s$ is the event that there is a fitness value $z \in [0..n]$ such that at least $\Delta$ individuals of the $s$-th population have fitness exactly $z$. 
		\item Otherwise (``thin $s$-th population'') let the event $B_s$ be true with probability $p \coloneqq n^{1-\Delta/2} (2\mu)^\Delta$ independent of all other random decisions of the algorithm. 
	\end{itemize}
	Since a random variable with binomial distribution with parameters $n$ and $\frac 12$ attains each value in $[0..n]$ with probability at most $2 /\sqrt{n}$, this follows from elementary estimates of binomial coefficient, see, e.g., \cite[Lemma~4.9]{2018arXiv180106733D}, a union bound over the $n+1$ possible values of $z$ and a similar estimate as above shows 
	\[\Pr[B_s] \le (n+1) \binom{\mu}{\Delta}2^\Delta n^{-\Delta/2} \le n^{1-\Delta/2} (2\mu)^\Delta = p\] 
	also in the first case, where the last estimate exploits that $\Delta \ge 2$. 
	
	Denote by $\overline B = \bigwedge_{s = 1}^t \overline B_s$ the event that none of the $B_s$ comes true. By a simple union bound, 
	\[\Pr[(X \le \tfrac n2) \wedge \overline B] \ge 1 - \exp(-\tfrac n8) - tp \ge 1 - \exp(- \tfrac n8) - \tfrac 14 n^{-c}.\]
	A simple induction shows that the event ``$(X \le \frac n2) \wedge \overline B$'' implies that in each iteration $s \in [2..t]$ at most those bit-positions which have been unsafe before can be converged to one. For $s=2$ this follows from the fact that all positions of the initial population are sampled randomly; consequently, the event $\overline B_2$ means that all fitness values occurred less than $\Delta$ times. This, however, implies that only a bit-position $i$ which was unsafe in the first iteration (that is, $X_{i1}=1$) can be converged in the second population. Since the total number of unsafe positions is at most $n/2$, also the second population is fat, that is, contains at least $n/2$ random bits. Repeating the previous arguments, we obtain that at most bit-positions which where unsafe at least once can be converged to one, and further, that all populations up to time $t$ are fat.
	
	Conditioning on the event ``$(X \le \frac n2) \wedge \overline B$'', we now regard the difference between the artificial process and a true run of the CSA. We have just seen that during the run of the artificial process, at least $tn/2$ times a bit-position was sampled randomly (without becoming unsafe). The number of ones in such a bit-position is described by a random variable $(Z \mid Z \le \mu-\Delta)$, where $Z$ follows a binomal law with parameters $\mu$ and $\tfrac 12$. In particular, with probability at least $2^{-\mu}$, this number is zero. Note that when a bit-position is sampled with zero ones, then the true process differs from the artificial process and the run of the CSA reaches a state from which it cannot generate the optimum of \OM. The probability that none of the at least $tn/2$ safe samplings of variables leads to this negative event is at most
	\begin{align*}
	&(1-2^{-\mu})^{tn/2} \le \exp\left(- \frac{2^{-\mu} t n}{2}\right)\\
	&\qquad\le \exp\left(- \frac{2^{-\mu}n}{2} \frac{2^\mu}{2\cdot 4(2\mu)^\Delta}\right) = \exp\left(-\frac{n}{16 (2\mu)^\Delta}\right).
	\end{align*}
	
	In summary, we see that with probability at least 
	\begin{align*}
	&\left(1 - \exp(-\tfrac n8) - \tfrac 14 n^{-c}\right) \left(1-\exp\left(- \frac{n}{16 (2\mu)^\Delta}\right)\right)\\
	&\hspace*{18 em}= 1-O(n^{-c}),
	\end{align*}
	the run of the true CAS fixes a position to zero.  
\end{proof}

\section{Conclusions}
\label{sec:conclusions}

We introduced the novel EDA \sigcGA, which optimizes both \OM and \LO in $O(n \log n)$ w.h.p. and in expectation. This is the first result of this kind for an EDA or even an EA. These run times are a result of the update process of the \sigcGA: it only updates its probabilistic model if it finds a significance in the history of its samples. In contrast, common EDAs or EAs that are analyzed theoretically do not store the entire history of their samples; EAs keep some samples as their population, and EDAs learn from samples iteratively and store the gained information implicitly in their model.

Since storing the entire history of samples demands a lot of memory if the \sigcGA runs longer, we proposed a method that stores the history compactly while maintaining its important information. We want to note that this method can be improved even further. Currently, the \sigcGA saves new data in each iteration, even if no information is gained. In order to further reduce the memory demands of the \sigcGA, it should save a bit only if it is different from that of the competing offspring, that is, if there actually was a bias in both offspring at that position. Note that this is more similar to how the cGA updates its frequencies. However, if a frequency of the \sigcGA is at~$1/2$, the number of samples that can contain important information (that is the pairs $(0, 1)$ and $(1, 0)$) is, in expectation, only half the number of all samples. Thus, the memory is only reduced by roughly a factor of~$2$.

All in all, the approach of the \sigcGA to reduce run times for a slight increase in memory appears to pay off very well.
In this first work, as often in the theory of evolutionary algorithms, we only regarded the two unimodal benchmark functions \OM and \LO. Since it has been observed, e.g., recently in~\cite{DoerrLMN17}, that insights derived from such analyses can lead to wrong conclusions for more difficult functions, an interesting next step would be to analyze the performance of the \sigcGA on objective functions that have true local optima or that have larger plateaus of equal fitness. Two benchmark functions have been suggested in this context, namely jump functions~\cite{DBLP:journals/tcs/DrosteJW02} having an easy to reach local optimum with a scalable basin of attraction and plateau functions~\cite{AntipovD18} having a plateau of scalable diameter around the optimum. We are vaguely optimistic that our \sigcGA has a good performance on these as well. We expect that the \sigcGA, as when optimizing \OM, quickly fixes a large number of bits to the correct value and then, different from classic EAs, profits from the fact that the missing bits are sampled with uniform distribution, leading to a much more efficient exploration~of the small subhypercube formed by these undecided bits. Needless to say, transforming this speculation into a formal proof would be a significant step forward to understanding the \sigcGA.

From a broader perspective, our work shows that by taking into account a longer history and only updating the model when the history justifies it, the performance of a classic EDA can be improved and its usability can be increased (since the difficult choice of the model update strength is now obsolete). An interesting question from this viewpoint would be to what extent similar ideas can be applied to other well-known EDAs.

From a very broad perspective, our work suggests that generally EC could profit from enriching the iterative evolutionary process with mechanisms that collect and exploit information over several iterations. So far, such learning-based concepts are rarely used in EC. The only theoretical works in this direction propose a history-based choice of the mutation strength~\cite{DoerrDY16ppsn} and analyze hyperheuristics that stick to a chosen subheuristic until its performance over the last $\tau$ iterations, $\tau$ a parameter of the algorithms, appears insufficient (see, e.g.,~\cite{DoerrLOW18} and the references therein).

\bibliographystyle{ACM-Reference-Format}
\bibliography{AdaptiveEDAs}


\begin{thebibliography}{44}

%%% the Web address.

\ifx \showCODEN    \undefined \def \showCODEN     #1{\unskip}     \fi
\ifx \showDOI      \undefined \def \showDOI       #1{#1}\fi
\ifx \showISBNx    \undefined \def \showISBNx     #1{\unskip}     \fi
\ifx \showISBNxiii \undefined \def \showISBNxiii  #1{\unskip}     \fi
\ifx \showISSN     \undefined \def \showISSN      #1{\unskip}     \fi
\ifx \showLCCN     \undefined \def \showLCCN      #1{\unskip}     \fi
\ifx \shownote     \undefined \def \shownote      #1{#1}          \fi
\ifx \showarticletitle \undefined \def \showarticletitle #1{#1}   \fi
\ifx \showURL      \undefined \def \showURL       {\relax}        \fi
% invisible to TeX
\providecommand\bibfield[2]{#2}
\providecommand\bibinfo[2]{#2}
\providecommand\natexlab[1]{#1}
\providecommand\showeprint[2][]{arXiv:#2}

\bibitem[\protect\citeauthoryear{Afshani, Agrawal, Doerr, Doerr, Larsen, and
  Mehlhorn}{Afshani et~al\mbox{.}}{2019}]%
        {AfshaniADDLM13}
\bibfield{author}{\bibinfo{person}{Peyman Afshani}, \bibinfo{person}{Manindra
  Agrawal}, \bibinfo{person}{Benjamin Doerr}, \bibinfo{person}{Carola Doerr},
  \bibinfo{person}{Kasper~Green Larsen}, {and} \bibinfo{person}{Kurt
  Mehlhorn}.} \bibinfo{year}{2019}\natexlab{}.
\newblock \showarticletitle{The query complexity of finding a hidden
  permutation}.
\newblock \bibinfo{journal}{\emph{Discrete Applied Mathematics}}
  (\bibinfo{year}{2019}), \bibinfo{pages}{28--50}.
\newblock
\urldef\tempurl%
\url{https://doi.org/10.1016/j.dam.2019.01.007}
\showDOI{\tempurl}


\bibitem[\protect\citeauthoryear{Anil and Wiegand}{Anil and Wiegand}{2009}]%
        {AnilW09}
\bibfield{author}{\bibinfo{person}{Gautham Anil} {and} \bibinfo{person}{R.~Paul
  Wiegand}.} \bibinfo{year}{2009}\natexlab{}.
\newblock \showarticletitle{Black-box search by elimination of fitness
  functions}. In \bibinfo{booktitle}{\emph{Proc. of FOGA'09}}.
  \bibinfo{publisher}{ACM}, \bibinfo{pages}{67--78}.
\newblock
\showISBNx{978-1-60558-414-0}
\urldef\tempurl%
\url{https://doi.org/10.1145/1527125.1527135}
\showDOI{\tempurl}


\bibitem[\protect\citeauthoryear{Antipov and Doerr}{Antipov and Doerr}{2018}]%
        {AntipovD18}
\bibfield{author}{\bibinfo{person}{Denis Antipov} {and}
  \bibinfo{person}{Benjamin Doerr}.} \bibinfo{year}{2018}\natexlab{}.
\newblock \showarticletitle{Precise runtime analysis for plateaus}. In
  \bibinfo{booktitle}{\emph{Proc. of PPSN'18}}. \bibinfo{publisher}{Springer},
  \bibinfo{pages}{117--128}.
\newblock
\urldef\tempurl%
\url{https://doi.org/10.1007/978-3-319-99259-4_10}
\showDOI{\tempurl}


\bibitem[\protect\citeauthoryear{Antipov, Doerr, Fang, and Hetet}{Antipov
  et~al\mbox{.}}{2018}]%
        {AntipovDFH18muLambdaEA}
\bibfield{author}{\bibinfo{person}{Denis Antipov}, \bibinfo{person}{Benjamin
  Doerr}, \bibinfo{person}{Jiefeng Fang}, {and} \bibinfo{person}{Tangi Hetet}.}
  \bibinfo{year}{2018}\natexlab{}.
\newblock \showarticletitle{A tight runtime analysis for the ({\(\mu\)} +
  {\(\lambda\)}) {EA}}. In \bibinfo{booktitle}{\emph{Proc. of GECCO'18}}.
  \bibinfo{publisher}{ACM}, \bibinfo{pages}{1459--1466}.
\newblock
\urldef\tempurl%
\url{https://doi.org/10.1145/3205455.3205627}
\showDOI{\tempurl}


\bibitem[\protect\citeauthoryear{Badkobeh, Lehre, and Sudholt}{Badkobeh
  et~al\mbox{.}}{2014}]%
        {DBLP:conf/ppsn/BadkobehLS14}
\bibfield{author}{\bibinfo{person}{Golnaz Badkobeh},
  \bibinfo{person}{Per~Kristian Lehre}, {and} \bibinfo{person}{Dirk Sudholt}.}
  \bibinfo{year}{2014}\natexlab{}.
\newblock \showarticletitle{Unbiased Black-Box Complexity of Parallel Search}.
  In \bibinfo{booktitle}{\emph{Proc.\ of PPSN'14}}.
  \bibinfo{publisher}{Springer}, \bibinfo{pages}{892--901}.
\newblock
\urldef\tempurl%
\url{https://doi.org/10.1007/978-3-319-10762-2_88}
\showDOI{\tempurl}


\bibitem[\protect\citeauthoryear{Dang, Lehre, and Nguyen}{Dang
  et~al\mbox{.}}{2019}]%
        {DangLN19UMDAOneMaxUpperBound}
\bibfield{author}{\bibinfo{person}{Duc{-}Cuong Dang},
  \bibinfo{person}{Per~Kristian Lehre}, {and} \bibinfo{person}{Phan Trung~Hai
  Nguyen}.} \bibinfo{year}{2019}\natexlab{}.
\newblock \showarticletitle{Level-Based Analysis of the Univariate Marginal
  Distribution Algorithm}.
\newblock \bibinfo{journal}{\emph{Algorithmica}} \bibinfo{volume}{81},
  \bibinfo{number}{2} (\bibinfo{year}{2019}), \bibinfo{pages}{668--702}.
\newblock
\urldef\tempurl%
\url{https://doi.org/10.1007/s00453-018-0507-5}
\showDOI{\tempurl}


\bibitem[\protect\citeauthoryear{Doerr}{Doerr}{2019}]%
        {Doerr19StochasticDomination}
\bibfield{author}{\bibinfo{person}{Benjamin Doerr}.}
  \bibinfo{year}{2019}\natexlab{}.
\newblock \showarticletitle{Analyzing randomized search heuristics via
  stochastic domination}.
\newblock \bibinfo{journal}{\emph{Theoretical Computer Science}}
  \bibinfo{volume}{773} (\bibinfo{year}{2019}), \bibinfo{pages}{115--137}.
\newblock
\urldef\tempurl%
\url{https://doi.org/10.1016/j.tcs.2018.09.024}
\showDOI{\tempurl}


\bibitem[\protect\citeauthoryear{Doerr}{Doerr}{2020}]%
        {2018arXiv180106733D}
\bibfield{author}{\bibinfo{person}{Benjamin Doerr}.}
  \bibinfo{year}{2020}\natexlab{}.
\newblock \showarticletitle{Probabilistic Tools for the Analysis of Randomized
  Optimization Heuristics}.
\newblock In \bibinfo{booktitle}{\emph{\emph{\cite{DoerrN20TheoryBook}}}}.
  \bibinfo{pages}{1--87}.
\newblock
\newblock
\shownote{Also available at \url{https://arxiv.org/abs/1801.06733}.}


\bibitem[\protect\citeauthoryear{Doerr and Doerr}{Doerr and Doerr}{2018}]%
        {DBLP:journals/algorithmica/DoerrD18}
\bibfield{author}{\bibinfo{person}{Benjamin Doerr} {and}
  \bibinfo{person}{Carola Doerr}.} \bibinfo{year}{2018}\natexlab{}.
\newblock \showarticletitle{Optimal Static and Self-Adjusting Parameter Choices
  for the (1+({\(\lambda\)}, {\(\lambda\)})) Genetic Algorithm}.
\newblock \bibinfo{journal}{\emph{Algorithmica}} \bibinfo{volume}{80},
  \bibinfo{number}{5} (\bibinfo{year}{2018}), \bibinfo{pages}{1658--1709}.
\newblock


\bibitem[\protect\citeauthoryear{Doerr, Doerr, and Yang}{Doerr
  et~al\mbox{.}}{2016}]%
        {DoerrDY16ppsn}
\bibfield{author}{\bibinfo{person}{Benjamin Doerr}, \bibinfo{person}{Carola
  Doerr}, {and} \bibinfo{person}{Jing Yang}.} \bibinfo{year}{2016}\natexlab{}.
\newblock \showarticletitle{$k$-bit mutation with self-adjusting $k$
  outperforms standard bit mutation}. In \bibinfo{booktitle}{\emph{Proc.\ of
  PPSN'16}}. \bibinfo{publisher}{Springer}, \bibinfo{pages}{824--834}.
\newblock
\urldef\tempurl%
\url{https://doi.org/10.1007/978-3-319-45823-6\_77}
\showDOI{\tempurl}


\bibitem[\protect\citeauthoryear{Doerr, Gie{\ss}en, Witt, and Yang}{Doerr
  et~al\mbox{.}}{2019}]%
        {DBLP:journals/algorithmica/DoerrGWY19}
\bibfield{author}{\bibinfo{person}{Benjamin Doerr}, \bibinfo{person}{Christian
  Gie{\ss}en}, \bibinfo{person}{Carsten Witt}, {and} \bibinfo{person}{Jing
  Yang}.} \bibinfo{year}{2019}\natexlab{}.
\newblock \showarticletitle{The $(1 + \lambda)$ Evolutionary Algorithm with
  Self-Adjusting Mutation Rate}.
\newblock \bibinfo{journal}{\emph{Algorithmica}} \bibinfo{volume}{81},
  \bibinfo{number}{2} (\bibinfo{year}{2019}), \bibinfo{pages}{593--631}.
\newblock
\urldef\tempurl%
\url{https://doi.org/10.1007/s00453-018-0502-x}
\showDOI{\tempurl}


\bibitem[\protect\citeauthoryear{Doerr, Jansen, Witt, and Zarges}{Doerr
  et~al\mbox{.}}{2013}]%
        {DoerrJWZ13}
\bibfield{author}{\bibinfo{person}{Benjamin Doerr}, \bibinfo{person}{Thomas
  Jansen}, \bibinfo{person}{Carsten Witt}, {and} \bibinfo{person}{Christine
  Zarges}.} \bibinfo{year}{2013}\natexlab{}.
\newblock \showarticletitle{A method to derive fixed budget results from
  expected optimisation times}. In \bibinfo{booktitle}{\emph{Proc.\ of
  GECCO'13}}. \bibinfo{publisher}{ACM}, \bibinfo{pages}{1581--1588}.
\newblock
\urldef\tempurl%
\url{https://doi.org/10.1145/2463372.2463565}
\showDOI{\tempurl}


\bibitem[\protect\citeauthoryear{Doerr and Johannsen}{Doerr and
  Johannsen}{2010}]%
        {DoerrJ10}
\bibfield{author}{\bibinfo{person}{Benjamin Doerr} {and}
  \bibinfo{person}{Daniel Johannsen}.} \bibinfo{year}{2010}\natexlab{}.
\newblock \showarticletitle{Edge-based representation beats vertex-based
  representation in shortest path problems}. In
  \bibinfo{booktitle}{\emph{Proc.\ of GECCO'10}}. \bibinfo{publisher}{ACM},
  \bibinfo{pages}{759--766}.
\newblock
\urldef\tempurl%
\url{https://doi.org/10.1145/1830483.1830618}
\showDOI{\tempurl}


\bibitem[\protect\citeauthoryear{Doerr and Künnemann}{Doerr and
  Künnemann}{2015}]%
        {DBLP:journals/tcs/DoerrK15}
\bibfield{author}{\bibinfo{person}{Benjamin Doerr} {and}
  \bibinfo{person}{Marvin Künnemann}.} \bibinfo{year}{2015}\natexlab{}.
\newblock \showarticletitle{Optimizing linear functions with the
  (1+{\(\lambda\)}) evolutionary algorithm -- different asymptotic runtimes for
  different instances}.
\newblock \bibinfo{journal}{\emph{Theoretical Computer Science}}
  \bibinfo{volume}{561} (\bibinfo{year}{2015}), \bibinfo{pages}{3--23}.
\newblock
\urldef\tempurl%
\url{https://doi.org/10.1016/j.tcs.2014.03.015}
\showDOI{\tempurl}


\bibitem[\protect\citeauthoryear{Doerr and Krejca}{Doerr and Krejca}{2018}]%
        {DoerrK18sigcGA}
\bibfield{author}{\bibinfo{person}{Benjamin Doerr} {and}
  \bibinfo{person}{Martin~S. Krejca}.} \bibinfo{year}{2018}\natexlab{}.
\newblock \showarticletitle{Significance-based estimation-of-distribution
  algorithms}. In \bibinfo{booktitle}{\emph{Proc. of GECCO'18}}.
  \bibinfo{publisher}{ACM}, \bibinfo{pages}{1483--1490}.
\newblock
\urldef\tempurl%
\url{https://doi.org/10.1145/3205455.3205553}
\showDOI{\tempurl}


\bibitem[\protect\citeauthoryear{Doerr, Le, Makhmara, and Nguyen}{Doerr
  et~al\mbox{.}}{2017}]%
        {DoerrLMN17}
\bibfield{author}{\bibinfo{person}{Benjamin Doerr}, \bibinfo{person}{Huu~Phuoc
  Le}, \bibinfo{person}{R\'egis Makhmara}, {and} \bibinfo{person}{Ta~Duy
  Nguyen}.} \bibinfo{year}{2017}\natexlab{}.
\newblock \showarticletitle{Fast genetic algorithms}. In
  \bibinfo{booktitle}{\emph{Proc.\ of GECCO'17}}. \bibinfo{publisher}{{ACM}},
  \bibinfo{pages}{777--784}.
\newblock
\urldef\tempurl%
\url{https://doi.org/10.1145/3205455.3205563}
\showDOI{\tempurl}


\bibitem[\protect\citeauthoryear{Doerr, Lissovoi, Oliveto, and Warwicker}{Doerr
  et~al\mbox{.}}{2018a}]%
        {DoerrLOW18}
\bibfield{author}{\bibinfo{person}{Benjamin Doerr}, \bibinfo{person}{Andrei
  Lissovoi}, \bibinfo{person}{Pietro~S. Oliveto}, {and}
  \bibinfo{person}{John~Alasdair Warwicker}.} \bibinfo{year}{2018}\natexlab{a}.
\newblock \showarticletitle{On the runtime analysis of selection
  hyper-heuristics with adaptive learning periods}. In
  \bibinfo{booktitle}{\emph{Proc. of GECCO'18}}. \bibinfo{publisher}{ACM},
  \bibinfo{pages}{1015--1022}.
\newblock
\urldef\tempurl%
\url{https://doi.org/10.1145/3205455.3205611}
\showDOI{\tempurl}


\bibitem[\protect\citeauthoryear{Doerr and Neumann}{Doerr and Neumann}{2020}]%
        {DoerrN20TheoryBook}
\bibfield{author}{\bibinfo{person}{Benjamin Doerr} {and} \bibinfo{person}{Frank
  Neumann}.} \bibinfo{year}{2020}\natexlab{}.
\newblock \bibinfo{booktitle}{\emph{Theory of Evolutionary Computation---Recent
  Developments in Discrete Optimization}}.
\newblock \bibinfo{publisher}{Springer}.
\newblock
\showISBNx{978-3-030-29414-4}
\urldef\tempurl%
\url{https://doi.org/10.1007/978-3-030-29414-4}
\showDOI{\tempurl}


\bibitem[\protect\citeauthoryear{Doerr, Neumann, Sudholt, and Witt}{Doerr
  et~al\mbox{.}}{2011}]%
        {DoerrNSW11OneAnt}
\bibfield{author}{\bibinfo{person}{Benjamin Doerr}, \bibinfo{person}{Frank
  Neumann}, \bibinfo{person}{Dirk Sudholt}, {and} \bibinfo{person}{Carsten
  Witt}.} \bibinfo{year}{2011}\natexlab{}.
\newblock \showarticletitle{Runtime analysis of the 1-{ANT} ant colony
  optimizer}.
\newblock \bibinfo{journal}{\emph{Theoretical Computer Science}}
  \bibinfo{volume}{412}, \bibinfo{number}{17} (\bibinfo{year}{2011}),
  \bibinfo{pages}{1629--1644}.
\newblock


\bibitem[\protect\citeauthoryear{Doerr and Winzen}{Doerr and Winzen}{2014}]%
        {DoerrW14ranking}
\bibfield{author}{\bibinfo{person}{Benjamin Doerr} {and}
  \bibinfo{person}{Carola Winzen}.} \bibinfo{year}{2014}\natexlab{}.
\newblock \showarticletitle{Ranking-based black-box complexity}.
\newblock \bibinfo{journal}{\emph{Algorithmica}}  \bibinfo{volume}{68}
  (\bibinfo{year}{2014}), \bibinfo{pages}{571--609}.
\newblock
\urldef\tempurl%
\url{https://doi.org/10.1007/s00453-012-9684-9}
\showDOI{\tempurl}


\bibitem[\protect\citeauthoryear{Doerr, Witt, and Yang}{Doerr
  et~al\mbox{.}}{2018b}]%
        {DBLP:conf/gecco/DoerrWY18}
\bibfield{author}{\bibinfo{person}{Benjamin Doerr}, \bibinfo{person}{Carsten
  Witt}, {and} \bibinfo{person}{Jing Yang}.} \bibinfo{year}{2018}\natexlab{b}.
\newblock \showarticletitle{Runtime analysis for self-adaptive mutation rates}.
  In \bibinfo{booktitle}{\emph{Proc.~of GECCO'18}}. \bibinfo{publisher}{ACM},
  \bibinfo{pages}{1475--1482}.
\newblock
\urldef\tempurl%
\url{https://doi.org/10.1145/3205455.3205569}
\showDOI{\tempurl}


\bibitem[\protect\citeauthoryear{Droste}{Droste}{2006}]%
        {DBLP:journals/nc/Droste06}
\bibfield{author}{\bibinfo{person}{Stefan Droste}.}
  \bibinfo{year}{2006}\natexlab{}.
\newblock \showarticletitle{A rigorous analysis of the compact genetic
  algorithm for linear functions}.
\newblock \bibinfo{journal}{\emph{Natural Computing}} \bibinfo{volume}{5},
  \bibinfo{number}{3} (\bibinfo{year}{2006}), \bibinfo{pages}{257--283}.
\newblock
\urldef\tempurl%
\url{https://doi.org/10.1007/s11047-006-9001-0}
\showDOI{\tempurl}


\bibitem[\protect\citeauthoryear{Droste, Jansen, and Wegener}{Droste
  et~al\mbox{.}}{2002}]%
        {DBLP:journals/tcs/DrosteJW02}
\bibfield{author}{\bibinfo{person}{Stefan Droste}, \bibinfo{person}{Thomas
  Jansen}, {and} \bibinfo{person}{Ingo Wegener}.}
  \bibinfo{year}{2002}\natexlab{}.
\newblock \showarticletitle{On the analysis of the {(1+1)} evolutionary
  algorithm}.
\newblock \bibinfo{journal}{\emph{Theoretical Computer Science}}
  \bibinfo{volume}{276}, \bibinfo{number}{1--2} (\bibinfo{year}{2002}),
  \bibinfo{pages}{51--81}.
\newblock
\urldef\tempurl%
\url{https://doi.org/10.1016/S0304-3975(01)00182-7}
\showDOI{\tempurl}


\bibitem[\protect\citeauthoryear{Droste, Jansen, and Wegener}{Droste
  et~al\mbox{.}}{2006}]%
        {DrosteJW06}
\bibfield{author}{\bibinfo{person}{Stefan Droste}, \bibinfo{person}{Thomas
  Jansen}, {and} \bibinfo{person}{Ingo Wegener}.}
  \bibinfo{year}{2006}\natexlab{}.
\newblock \showarticletitle{Upper and lower bounds for randomized search
  heuristics in black-box optimization}.
\newblock \bibinfo{journal}{\emph{Theory of Computing Systems}}
  \bibinfo{volume}{39} (\bibinfo{year}{2006}), \bibinfo{pages}{525--544}.
\newblock
\urldef\tempurl%
\url{https://doi.org/10.1007/s00224-004-1177-z}
\showDOI{\tempurl}


\bibitem[\protect\citeauthoryear{Friedrich, K{\"{o}}tzing, and
  Krejca}{Friedrich et~al\mbox{.}}{2016}]%
        {DBLP:conf/gecco/FriedrichKK16}
\bibfield{author}{\bibinfo{person}{Tobias Friedrich}, \bibinfo{person}{Timo
  K{\"{o}}tzing}, {and} \bibinfo{person}{Martin~S. Krejca}.}
  \bibinfo{year}{2016}\natexlab{}.
\newblock \showarticletitle{{EDA}s cannot be balanced and stable}. In
  \bibinfo{booktitle}{\emph{Proc. of GECCO'16}}. \bibinfo{publisher}{ACM},
  \bibinfo{pages}{1139--1146}.
\newblock
\urldef\tempurl%
\url{https://doi.org/10.1145/2908812.2908895}
\showDOI{\tempurl}


\bibitem[\protect\citeauthoryear{Gie{\ss}en and Witt}{Gie{\ss}en and
  Witt}{2017}]%
        {GiessenW17}
\bibfield{author}{\bibinfo{person}{Christian Gie{\ss}en} {and}
  \bibinfo{person}{Carsten Witt}.} \bibinfo{year}{2017}\natexlab{}.
\newblock \showarticletitle{The interplay of population size and mutation
  probability in the {(1} + {\(\lambda\)}) {EA} on {OneMax}}.
\newblock \bibinfo{journal}{\emph{Algorithmica}}  \bibinfo{volume}{78}
  (\bibinfo{year}{2017}), \bibinfo{pages}{587--609}.
\newblock
\urldef\tempurl%
\url{https://doi.org/10.1007/s00453-016-0214-z}
\showDOI{\tempurl}


\bibitem[\protect\citeauthoryear{Harik, Lobo, and Goldberg}{Harik
  et~al\mbox{.}}{1999}]%
        {HarikLG98}
\bibfield{author}{\bibinfo{person}{Georges~R. Harik},
  \bibinfo{person}{Fernando~G. Lobo}, {and} \bibinfo{person}{David~E.
  Goldberg}.} \bibinfo{year}{1999}\natexlab{}.
\newblock \showarticletitle{The compact genetic algorithm}.
\newblock \bibinfo{journal}{\emph{{IEEE} Transactions on Evolutionary
  Computation}} \bibinfo{volume}{3}, \bibinfo{number}{4}
  (\bibinfo{year}{1999}), \bibinfo{pages}{287--297}.
\newblock


\bibitem[\protect\citeauthoryear{Jansen, {De Jong}, and Wegener}{Jansen
  et~al\mbox{.}}{2005}]%
        {JansenJW05}
\bibfield{author}{\bibinfo{person}{Thomas Jansen}, \bibinfo{person}{Kenneth~A.
  {De Jong}}, {and} \bibinfo{person}{Ingo Wegener}.}
  \bibinfo{year}{2005}\natexlab{}.
\newblock \showarticletitle{On the choice of the offspring population size in
  evolutionary algorithms}.
\newblock \bibinfo{journal}{\emph{Evolutionary Computation}}
  \bibinfo{volume}{13} (\bibinfo{year}{2005}), \bibinfo{pages}{413--440}.
\newblock
\urldef\tempurl%
\url{https://doi.org/10.1162/106365605774666921}
\showDOI{\tempurl}


\bibitem[\protect\citeauthoryear{Jansen and Zarges}{Jansen and Zarges}{2014}]%
        {JansenZ14}
\bibfield{author}{\bibinfo{person}{Thomas Jansen} {and}
  \bibinfo{person}{Christine Zarges}.} \bibinfo{year}{2014}\natexlab{}.
\newblock \showarticletitle{Performance analysis of randomised search
  heuristics operating with a fixed budget}.
\newblock \bibinfo{journal}{\emph{Theoretical Computer Science}}
  \bibinfo{volume}{545} (\bibinfo{year}{2014}), \bibinfo{pages}{39--58}.
\newblock
\urldef\tempurl%
\url{https://doi.org/10.1016/j.tcs.2013.06.007}
\showDOI{\tempurl}


\bibitem[\protect\citeauthoryear{Krejca and Witt}{Krejca and Witt}{2017}]%
        {DBLP:conf/foga/KrejcaW17}
\bibfield{author}{\bibinfo{person}{Martin~S. Krejca} {and}
  \bibinfo{person}{Carsten Witt}.} \bibinfo{year}{2017}\natexlab{}.
\newblock \showarticletitle{Lower bounds on the run time of the univariate
  marginal distribution algorithm on OneMax}. In
  \bibinfo{booktitle}{\emph{Proc. of FOGA'17}}. \bibinfo{publisher}{ACM},
  \bibinfo{pages}{65--79}.
\newblock
\urldef\tempurl%
\url{https://doi.org/10.1145/3040718.3040724}
\showDOI{\tempurl}


\bibitem[\protect\citeauthoryear{Krejca and Witt}{Krejca and Witt}{2020}]%
        {KrejcaW18EDABookChapter}
\bibfield{author}{\bibinfo{person}{Martin~S. Krejca} {and}
  \bibinfo{person}{Carsten Witt}.} \bibinfo{year}{2020}\natexlab{}.
\newblock \showarticletitle{Theory of Estimation-of-Distribution Algorithms}.
\newblock In \bibinfo{booktitle}{\emph{\emph{\cite{DoerrN20TheoryBook}}}}.
  \bibinfo{pages}{405--442}.
\newblock
\newblock
\shownote{Also available at \url{http://arxiv.org/abs/1806.05392}.}


\bibitem[\protect\citeauthoryear{Lehre and Nguyen}{Lehre and Nguyen}{2018}]%
        {LehreN18PBIL}
\bibfield{author}{\bibinfo{person}{Per~Kristian Lehre} {and}
  \bibinfo{person}{Phan Trung~Hai Nguyen}.} \bibinfo{year}{2018}\natexlab{}.
\newblock \showarticletitle{Level-based analysis of the population-based
  incremental learning algorithm}. In \bibinfo{booktitle}{\emph{Proc. of
  PPSN'18}}. \bibinfo{publisher}{Springer}, \bibinfo{pages}{105--116}.
\newblock
\urldef\tempurl%
\url{https://doi.org/10.1007/978-3-319-99259-4_9}
\showDOI{\tempurl}


\bibitem[\protect\citeauthoryear{Lehre and Witt}{Lehre and Witt}{2012}]%
        {DBLP:journals/algorithmica/LehreW12}
\bibfield{author}{\bibinfo{person}{Per~Kristian Lehre} {and}
  \bibinfo{person}{Carsten Witt}.} \bibinfo{year}{2012}\natexlab{}.
\newblock \showarticletitle{Black-box search by unbiased variation}.
\newblock \bibinfo{journal}{\emph{Algorithmica}} \bibinfo{volume}{64},
  \bibinfo{number}{4} (\bibinfo{year}{2012}), \bibinfo{pages}{623--642}.
\newblock
\urldef\tempurl%
\url{https://doi.org/10.1007/s00453-012-9616-8}
\showDOI{\tempurl}


\bibitem[\protect\citeauthoryear{Lengler, Sudholt, and Witt}{Lengler
  et~al\mbox{.}}{2018}]%
        {LenglerSW18cGAMediumSteps}
\bibfield{author}{\bibinfo{person}{Johannes Lengler}, \bibinfo{person}{Dirk
  Sudholt}, {and} \bibinfo{person}{Carsten Witt}.}
  \bibinfo{year}{2018}\natexlab{}.
\newblock \showarticletitle{Medium step sizes are harmful for the compact
  genetic algorithm}. In \bibinfo{booktitle}{\emph{Proc. of GECCO'18}}.
  \bibinfo{publisher}{ACM}, \bibinfo{pages}{1499--1506}.
\newblock
\urldef\tempurl%
\url{https://doi.org/10.1145/3205455.3205576}
\showDOI{\tempurl}


\bibitem[\protect\citeauthoryear{Moraglio and Sudholt}{Moraglio and
  Sudholt}{2017}]%
        {DBLP:journals/ec/MoraglioS17}
\bibfield{author}{\bibinfo{person}{Alberto Moraglio} {and}
  \bibinfo{person}{Dirk Sudholt}.} \bibinfo{year}{2017}\natexlab{}.
\newblock \showarticletitle{Principled design and runtime analysis of abstract
  convex evolutionary search}.
\newblock \bibinfo{journal}{\emph{Evolutionary Computation}}
  \bibinfo{volume}{25}, \bibinfo{number}{2} (\bibinfo{year}{2017}),
  \bibinfo{pages}{205--236}.
\newblock
\urldef\tempurl%
\url{https://doi.org/10.1162/EVCO_a_00169}
\showDOI{\tempurl}


\bibitem[\protect\citeauthoryear{Neumann, Sudholt, and Witt}{Neumann
  et~al\mbox{.}}{2009}]%
        {DBLP:journals/swarm/NeumannSW09}
\bibfield{author}{\bibinfo{person}{Frank Neumann}, \bibinfo{person}{Dirk
  Sudholt}, {and} \bibinfo{person}{Carsten Witt}.}
  \bibinfo{year}{2009}\natexlab{}.
\newblock \showarticletitle{Analysis of different {MMAS} {ACO} algorithms on
  unimodal functions and plateaus}.
\newblock \bibinfo{journal}{\emph{Swarm Intelligence}} \bibinfo{volume}{3},
  \bibinfo{number}{1} (\bibinfo{year}{2009}), \bibinfo{pages}{35--68}.
\newblock
\urldef\tempurl%
\url{https://doi.org/10.1007/s11721-008-0023-3}
\showDOI{\tempurl}


\bibitem[\protect\citeauthoryear{Neumann and Witt}{Neumann and Witt}{2009}]%
        {DBLP:journals/algorithmica/NeumannW09}
\bibfield{author}{\bibinfo{person}{Frank Neumann} {and}
  \bibinfo{person}{Carsten Witt}.} \bibinfo{year}{2009}\natexlab{}.
\newblock \showarticletitle{Runtime analysis of a simple ant colony
  optimization algorithm}.
\newblock \bibinfo{journal}{\emph{Algorithmica}} \bibinfo{volume}{54},
  \bibinfo{number}{2} (\bibinfo{year}{2009}), \bibinfo{pages}{243--255}.
\newblock
\urldef\tempurl%
\url{https://doi.org/10.1007/s00453-007-9134-2}
\showDOI{\tempurl}


\bibitem[\protect\citeauthoryear{Oliveto and Witt}{Oliveto and Witt}{2011}]%
        {DBLP:journals/algorithmica/OlivetoW11}
\bibfield{author}{\bibinfo{person}{Pietro~S. Oliveto} {and}
  \bibinfo{person}{Carsten Witt}.} \bibinfo{year}{2011}\natexlab{}.
\newblock \showarticletitle{Simplified drift analysis for proving lower bounds
  in evolutionary computation}.
\newblock \bibinfo{journal}{\emph{Algorithmica}} \bibinfo{volume}{59},
  \bibinfo{number}{3} (\bibinfo{year}{2011}), \bibinfo{pages}{369--386}.
\newblock
\urldef\tempurl%
\url{https://doi.org/10.1007/s00453-010-9387-z}
\showDOI{\tempurl}


\bibitem[\protect\citeauthoryear{Oliveto and Witt}{Oliveto and Witt}{2012}]%
        {DBLP:journals/corr/OlivetoW12}
\bibfield{author}{\bibinfo{person}{Pietro~S. Oliveto} {and}
  \bibinfo{person}{Carsten Witt}.} \bibinfo{year}{2012}\natexlab{}.
\newblock \showarticletitle{Erratum: simplified drift analysis for proving
  lower bounds in evolutionary computation}.
\newblock \bibinfo{journal}{\emph{CoRR}}  \bibinfo{volume}{abs/1211.7184}
  (\bibinfo{year}{2012}).
\newblock
\urldef\tempurl%
\url{http://arxiv.org/abs/1211.7184}
\showURL{%
\tempurl}


\bibitem[\protect\citeauthoryear{Pelikan, Hauschild, and Lobo}{Pelikan
  et~al\mbox{.}}{2015}]%
        {PelikanHandbook15}
\bibfield{author}{\bibinfo{person}{Martin Pelikan}, \bibinfo{person}{Mark
  Hauschild}, {and} \bibinfo{person}{Fernando~G. Lobo}.}
  \bibinfo{year}{2015}\natexlab{}.
\newblock \showarticletitle{Estimation of distribution algorithms}.
\newblock In \bibinfo{booktitle}{\emph{Springer Handbook of Computational
  Intelligence}}. \bibinfo{pages}{899--928}.
\newblock
\urldef\tempurl%
\url{https://doi.org/10.1007/978-3-662-43505-2_45}
\showDOI{\tempurl}


\bibitem[\protect\citeauthoryear{Sudholt and Witt}{Sudholt and Witt}{2019}]%
        {SudholtW19cGAACOOneMaxLowerBound}
\bibfield{author}{\bibinfo{person}{Dirk Sudholt} {and} \bibinfo{person}{Carsten
  Witt}.} \bibinfo{year}{2019}\natexlab{}.
\newblock \showarticletitle{On the Choice of the Update Strength in
  Estimation-of-Distribution Algorithms and Ant Colony Optimization}.
\newblock \bibinfo{journal}{\emph{Algorithmica}} \bibinfo{volume}{81},
  \bibinfo{number}{4} (\bibinfo{year}{2019}), \bibinfo{pages}{1450--1489}.
\newblock
\urldef\tempurl%
\url{https://doi.org/10.1007/s00453-018-0480-z}
\showDOI{\tempurl}


\bibitem[\protect\citeauthoryear{Witt}{Witt}{2006}]%
        {Witt06SimplePseudoBool}
\bibfield{author}{\bibinfo{person}{Carsten Witt}.}
  \bibinfo{year}{2006}\natexlab{}.
\newblock \showarticletitle{Runtime analysis of the ($\mu$ + 1) {EA} on simple
  pseudo-{B}oolean functions}.
\newblock \bibinfo{journal}{\emph{Evolutionary Computation}}
  \bibinfo{volume}{14} (\bibinfo{year}{2006}), \bibinfo{pages}{65--86}.
\newblock
\urldef\tempurl%
\url{https://doi.org/10.1162/evco.2006.14.1.65}
\showDOI{\tempurl}


\bibitem[\protect\citeauthoryear{Witt}{Witt}{2019}]%
        {DBLP:journals/algorithmica/Witt19}
\bibfield{author}{\bibinfo{person}{Carsten Witt}.}
  \bibinfo{year}{2019}\natexlab{}.
\newblock \showarticletitle{Upper Bounds on the Running Time of the Univariate
  Marginal Distribution Algorithm on OneMax}.
\newblock \bibinfo{journal}{\emph{Algorithmica}} \bibinfo{volume}{81},
  \bibinfo{number}{2} (\bibinfo{year}{2019}), \bibinfo{pages}{632--667}.
\newblock
\urldef\tempurl%
\url{https://doi.org/10.1007/s00453-018-0463-0}
\showDOI{\tempurl}


\bibitem[\protect\citeauthoryear{Zheng, Yang, and Doerr}{Zheng
  et~al\mbox{.}}{2018}]%
        {ZhengYD18}
\bibfield{author}{\bibinfo{person}{Weijie Zheng}, \bibinfo{person}{Guangwen
  Yang}, {and} \bibinfo{person}{Benjamin Doerr}.}
  \bibinfo{year}{2018}\natexlab{}.
\newblock \showarticletitle{Working principles of binary differential
  evolution}. In \bibinfo{booktitle}{\emph{Proc.\ of GECCO'18}}.
  \bibinfo{publisher}{ACM}, \bibinfo{pages}{1103--1110}.
\newblock
\urldef\tempurl%
\url{https://doi.org/10.1145/3205455.3205623}
\showDOI{\tempurl}


\end{thebibliography}

\end{document}